\documentclass{article}
\usepackage{ijcai22}

\usepackage{times}

\usepackage{soul}
\usepackage{url}
\usepackage[hidelinks]{hyperref}
\usepackage[utf8]{inputenc}
\usepackage[small]{caption}
\usepackage{graphicx}
\usepackage{amsmath}
\usepackage{booktabs}
\usepackage{amsfonts}
\usepackage{amssymb}
\usepackage{times}
\usepackage{soul}
\usepackage{url}
\usepackage[hidelinks]{hyperref}
\usepackage[utf8]{inputenc}
\usepackage{graphicx}
\usepackage{amsmath}
\usepackage{amsthm}
\usepackage{booktabs}
\usepackage{algorithm}
\usepackage{algorithmic}
\usepackage{tablefootnote}
\usepackage{threeparttable}
\usepackage{dblfloatfix}
\usepackage{amsthm}
\usepackage{stmaryrd}
\usepackage{mathtools}
\usepackage{color}
\usepackage{comment}
\usepackage{algorithmic}
\usepackage{algorithm}
\usepackage{color}
\usepackage{tablefootnote}
\usepackage{threeparttable}
\newtheorem{theorem}{Theorem}
\newtheorem*{theorem*}{Theorem}
\newtheorem{lemma}{Lemma}
\newtheorem{remark}{Remark}

\usepackage{comment}

\author{
Anshuka Rangi$^1$
\and
Haifeng Xu$^2$\and
Long Tran-Thanh$^3$\and
Massimo Franceschetti$^1$
\affiliations
$^1$ University of California San Diego, USA\\
$^2$University of Virginia, USA\\
$^3$University of Warwick, UK
\emails
\{arangi, mfranceschetti\}@ucsd.edu,
hx4ad@virginia.edu,
long.tran-thanh@warwick.ac.uk
}

\title{Understanding the Limits of   Poisoning Attacks in  Episodic Reinforcement Learning}

\begin{document}

\maketitle
\begin{abstract}
To understand the security threats to reinforcement learning (RL) algorithms,  this paper studies poisoning attacks to manipulate \emph{any} order-optimal learning algorithm towards a targeted policy in episodic RL and examines the potential damage  of two natural types of  poisoning attacks, i.e., the manipulation of \emph{reward} and \emph{action}. We discover that the effect  of attacks crucially depend on whether the rewards are bounded or unbounded. In bounded reward  settings, we  show that only reward manipulation or only action manipulation cannot guarantee a successful attack. 
However,  by combining   reward and action manipulation,  the adversary can  manipulate any order-optimal learning algorithm to follow any targeted policy  with $\tilde{\Theta}(\sqrt{T})$ total attack cost, which is order-optimal,
without  any knowledge of the underlying MDP.
In contrast, in unbounded reward  settings, we show that reward manipulation attacks are sufficient for an adversary to successfully  manipulate any order-optimal learning algorithm to follow any targeted policy  using $\tilde{O}(\sqrt{T})$ amount of contamination. Our results reveal useful insights about what can or cannot be achieved by   poisoning attacks, and are set to spur more works on the design of robust RL algorithms. 
\end{abstract}

\section{Introduction}
Learning algorithms have been widely used in web services \cite{zhao2018deep},  conversational AI \cite{dhingra2016towards},  UAV coordination \cite{venugopal2021reinforcement},   medical trials \cite{badanidiyuru2018bandits,rangi2019online}, and   crowdsourcing systems \cite{rangi2018multi}. The distributed nature of these applications makes these algorithms prone to third party attacks. For example, in web services  decision making critically depends on   reward collection, and this is prone to attacks that can impact observations and monitoring, delay or temper rewards, produce link failures, and generally modify or delete information through hijacking of communication links   \cite{agarwal2016making,cardenas2008secure}. { Making} these systems secure requires an understanding of the regime where the systems may be vulnerable, and designing ways to mitigate these attacks. This paper focuses on the former aspect, namely  understanding of the regime where the systems can be attacked, in   episodic Reinforcement Learning (RL). 
%Hence, it is crucial to study the regime where MAB algorithms can be attacked.  

%and sensor network, the reward collection can be hindered by loss of packet, link failures, interruptions by extern interruption of  %Specifically, we target failure modes and technical debts incurred by: feedback loops and bias, distributed data collection, changes in the environment, and weak monitoring and debugging

We consider a \emph{man in the middle} (MITM) attack. In this attack, there are three entities: the environment, the learner (RL algorithm), and the adversary. The learner interacts with the environment for $T$ episodes, and each episode has $H$ steps. 
In episode $t\leq T$ at step $h\leq H$, the learner observes the state $s_t(h)\in\mathcal{S}$ of the environment, selects an action $a_t(h) \in \mathcal{A}$, 
%, that is an element of a set of cardinality $K$, 
the environment then generates a reward $r_t(s_t(h),a_t(h))$ and changes its state based on an underlying Markov Decision Process (MDP), and attempts to communicate the new state to the learner. However, an adversary acts as a ``man in the middle'' between the learner and the environment. It can observe and may \emph{manipulate the action} $a_t(h)$ to $a^o_t(h)\in \mathcal{A}$ which will generate reward $r_t(s_t(h),a^o_t(h))$ corresponding to the manipulated action. Additionally, the adversary may also \emph{intercept the reward   $r_t(s_t(h),a^o_t(h))$} by adding contamination noise {$\epsilon_{t,h}(s_t(h),a_t(h))$}. With both attacks,  the learner ends up observing the contaminated reward $r^o_t(s_t(h),a_t(h))=r_t(s_t(h),a^o_t(h))+\epsilon_{t,h}(s_t(h),a_t(h))$. The cost of attack is measured as the \emph{amount of contamination} $\sum_{t,h}|\epsilon_{t,h}(s_t(h),a_t(h))|$ and   \emph{number of action manipulations} $\sum_{t,h}\mathbf{1}(a_t(h)\neq a^o_t(h))$, respectively.    Notably, with the wide application of RL today, MITM attack is a realistic concern to the   vulnerability of RL algorithms and is thus important to understand. For instance, RL-based UAV coordination to reduce poaching activities in conservation areas is naturally subject to poachers' poisoning attacks, which can falsify the reward feedback   (i.e., reward manipulation) and executed actions (i.e., action manipulation) \cite{venugopal2021reinforcement}; similarly, RL algorithms for recommender systems are subject to   attacks from hackers or competitors \cite{zhao2018deep}.   
%in smart traffic control with RL, , appliance management in smart homes, or RL based UAV coordination in disaster response, both data and control commands are sent via wireless communication channels. In these settings it is possible to launch a man-in-the-middle attack to manipulate the content of these wireless traffic. This attack can falsify the report on both the observations (i.e.,reward manipulation) and executed actions of the UAV/appliance/smart traffic light (i.e., action manipulation).  %Another motivational example comes from the wireless sensor based surveillance domain, where the attacker may physically replace/corrupt some sensors and report fake observations or execute wrong actions of those sensors.

\begin{table*}[!h]
\centering
\begin{threeparttable}

\begin{tabular}{ |p{4.5cm}|p{2cm}|p{5cm}|p{3.75cm}| } 
 \hline
 Settings & Reward & Attack & Bound on Attack Cost \\ 
 
 \hline
 % Black-box and White-Box in MAB & Unbounded & Reward Manipulation & $O(\log T)$  \qquad(\cite{liu2019data}) \\
 %  \hline
   
  %Black-box and White-Box in MAB & Unbounded & Action Manipulation & $O(\log T)$  \qquad(\cite{liu2020action}) \\
  % \hline
 White-box in  infinite horizon \tnote{2}\tablefootnote{ Performance of algorithms in infinite horizon RL is studied by analyzing the performance over first $T$ rounds  \cite{wu2021nearly} } RL& Unbounded &Reward Manipulation & $\tilde O(\sqrt{T})$ \qquad\qquad\qquad\cite{rakhsha2020policy}\tnote{3}\\
 \hline
 White-box in  {infinite horizon} RL& Unbounded &Dynamics Manipulation ( under sufficient conditions only)& $\tilde O(\sqrt{T})$ \qquad\qquad\qquad \cite{rakhsha2020policy}\tnote{3}\footnote{This result is not explicitly stated, and is derived for order-optimal algorithms from the reference}\\
 %&  &( under sufficient conditions) & \cite{rakhsha2020policy}&\cite{rakhsha2020policy}\\
 \hline
Black-box in  infinite horizon RL&  {Unbounded} & {Reward Manipulation ( under a setting with $L$ learner)}&  $\tilde{O}(T\log L+L\sqrt{T})$ \qquad\qquad\qquad \cite{rakhsha2021reward}\tnote{3}\\
 \hline
  Black-box in episodic RL& Unbounded &Reward Manipulation & $\tilde{O}(\sqrt{T})$ \quad(This work)\\
  \hline
  Black-box and White Box in  RL& Bounded &Reward Manipulation & Infeasible (This work)\\
 \hline
 Black-box and White Box in  RL& Bounded &Action Manipulation & Infeasible (This work)\\
  \hline
   Black-box and White Box in RL & Bounded &Reward and Action Manipulation & $\tilde{\Theta}(\sqrt{T})$ \quad(This work) \\
 \hline
\end{tabular}
\vskip -8pt
\caption{Comparison of the attack cost in the episodic RL and infinite RL setting when the adversary does nott know the RL algorithm.}
\label{table:comparison_1}
\begin{tablenotes}
\small{
\item[2] Performance of algorithms in infinite horizon RL is studied by analyzing the performance over first $T$ rounds \cite{wu2021nearly} 
\item[3] This result is not explicitly stated, and is derived for order-optimal algorithms from the reference.}
\end{tablenotes}
\end{threeparttable}
\vskip -13pt
\end{table*}

Reward poisoning attack is a special case of the MITM attack where $a^o_t(h)=a_t(h)$, and has been widely studied   in both RL and Multi-Armed Bandits (MAB) settings \cite{jun2018adversarial,rakhsha2020policy,rangi2022saving}.
Likewise, action manipulation attack is another special case of the MITM attack where $\epsilon_{t,h}(s_t(h),a_t(h))=0$, and has been previously studied for MAB setting \cite{liu2020action}. Another variant of action manipulation attack, recently studied in RL \cite{rakhsha2020policy}, is   manipulation of the transition dynamics. This can be  considered as manipulating the action $a_t(h)$ to another action,   not necessarily in $\mathcal{A}$. Therefore, the attacker has strictly stronger power there than the action manipulation as in our setting. MITM attacks has also been previously considered in cyber-physical systems \cite{rangi2020learning,khojasteh2020learning}. %Given that RL algorithms are increasingly used in critical applications, including cyber-physical systems \cite{li2019reinforcement}, it is of utmost importance to investigate the security threat to RL algorithms against different forms of poisoning attacks.

%\hf{ my rephrase of the paragraph, please check. 

We consider MITM attacks in two different settings:  \emph{unbounded} rewards and \emph{bounded} rewards, which turns out to differ fundamentally. In unbounded reward setting, the contamination $\epsilon_{t,h}(s_t(h),a_t(h))$ is unconstrained whereas in bounded reward setting, the contaminated reward $r^o_t(s_t(h),a^o_t(h))$ is constrained to be in the interval $[0,1]$, just like the original rewards $r_t(s_t(h),a^o_t(h))$. This constrained situation limits the attacker's contamination at every round, and turns out to be provably  more difficult to attack. In each setting, we study attack strategy in more realistic black-box setting in which the attacker does not know, and needs to learn, the underlying MDP as well. A similar study between the two settings has been performed for MAB in \cite{rangi2022saving}, where it is shown that bounded rewards setting is more difficult to attack in comparison to unbounded reward setting. This paper extends the study from MAB to RL, and provides new insights on the feasibility of attacks. 

\subsection{Summary of Contributions}
We consider poisoning attacks with the objective of forcing the learner to execute a target policy $\pi^+$. More specifically,  for all $h\leq H$ and $s\in \mathcal{S}$,  if $\pi_h(s)\neq \pi_h^+(s)$, then the attack aims to induce values satisfying 
\begin{equation}\label{eq:objAttack1}
    \tilde{V}_h^{\pi}(s)< \tilde{V}_h^{\pi^+}(s),
\end{equation}
where  $a)$ policy $\pi$ of an agent is a collection of $H$ functions $\{\pi_h:\mathcal{S}\to\mathcal{A}\}$; and $b)$ value function $\tilde{V}_h^\pi(s)$ is the expected reward of state $s$ under policy $\pi$ using contaminated reward observation, between step $h$  until step $H$.

%Our first main result is a reward poisoning attack in the black-box and unbounded reward setting which provably leads to any $O(\sqrt{T})$ regret learning algorithm to play a target policy $\pi^+$ with $\tilde{O}(\sqrt{T})$ amount of contaminations. We point out another relevant though not fully comparably result   recently developed by \cite{rakhsha2020policy}. In particular, they study \emph{white-box  joint attacks} of both reward and transition manipulation in the infinite-horizon MDPs and also  proposed an $\tilde{O}(\sqrt{T})$ attack based on optimization techniques. We are in the more challenging \emph{black-box} attack setting with only reward manipulation but for an episodic MDP. It was a surprise to us that $\tilde{O}(\sqrt{T})$ amount of acontamination turns out to still suffice.  % This white-box attack is order optimal in the amount of contamination. Extending this idea to episodic RL, we propose , which is order-optimal upto logarithmic factor. However, 
%We also point out that a variant of  the black box attack by \cite{rakhsha2021reward} for policy teaching may be applied to our setting as well but will require  $\tilde{O}(T^{3/4})$ amount of contamination and thus is less effective than us. 

%\ar{Moved the constraint on attack here since we contaminate all actions in the initialization phase in black-box setting.}
{ We propose the first set of efficient \emph{black-box} online attacks to any order-optimal episodic RL algorithms in both bounded and unbounded reward settings.}  Specifically, for the bounded reward setting,  % Similar to   \cite{rakhsha2020policy},  under the constraints that the reward manipulation, namely $\epsilon_{t,h}(s_t(h),a_t(h))\neq 0$, and action manipulation, namely $a_t(h)\neq a^o_t(h)$, can occur only if the selected action is different from the desired action, namely $a_t(h)\neq \pi^+_h(s_t(h))$. 
 we show that mere reward manipulation   does not guarantee successful attacks; namely there exist an MDP and a target policy $\pi^+$  which cannot be attacked --- i.e., \eqref{eq:objAttack1} cannot be achieved --- by manipulating only the rewards.  %\arr{under the constraint that attacker can manipulate rewards whenever the learner's action is not in accordance with the attacker's desired policy $\pi^+$.} \hf{probably no need to emphasize it here} 
 Similarly, we show that action manipulation attack does not suffice by showing the existence of  an MDP and a target policy $\pi^+$ which cannot be attacked by manipulating   actions. Hence, to guarantee a successful attack,  the attacker needs a combined the power of reward and action manipulation. Indeed, we propose an MITM attack in bounded reward setting, which requires $\tilde{O}(\sqrt{T})$ \footnote{Throughout the paper, $\tilde{\Theta}(\cdot)$ and $\tilde{O}(\cdot)$ notation omits the  logarithmic terms. } amount of reward contamination and $\tilde{O}(\sqrt{T})$ number of action manipulations to attack any order-optimal learning algorithm. We also show that this attack cost, namely sum of amount of reward contamination and  number of action manipulations,  is order-optimal. % using results in  \cite{chen2021improved}.
% In the proposed attack, the attacker uses both reward manipulation and action manipulation together whenever the learner's action is not in accordance with the attacker's desired policy $\pi^+$. % \hf{Is this also order-optimal? May want to say a bit about our techniques --- TBD} \ar{Every result is for order-optimal algorithms.} \ar{I have added it. Can you check?}
%In comparison to the literature in MAB, it is already established that the amount of contamination needed to attack bounded reward setting is $\Theta(\log T)$, whereas  the amount of contamination needed to attack unbounded reward setting is $\Theta(\sqrt{\log T})$ \cite{rangi2022saving,zuo2020near}. 

Next we move to unbounded reward settings. Reward manipulation attack for unbounded rewards has been studied  in \cite{rakhsha2020policy}. However, the work investigated \emph{infinite-horizon} RL and proposed a  \emph{white-box}  attack. %This \emph{white-box} attack is order optimal in the amount of contamination. 
Extending this research agenda, we propose \emph{black-box} attack for \emph{episodic} RL, and show that our proposed attack can attack any order-optimal learning algorithm in $\tilde{O}(\sqrt{T})$ amount of contamination. % \checkr{ which is order-optimal upto logarithmic factor.} \ar{How?} %\hf{what about our result for action manipulation not sufficient?} \ar{Since that is also studied under the constraint, we should mention that in the next paragraph. } %However, a variant of  the black box attack setting studied in \cite{rakhsha2021reward} for policy teaching if applied to our setting requires $\tilde{O}(T^{3/4})$ amount of contamination. 
An interesting conceptual message from our results is that  bounded reward setting is more difficult to attack than the unbounded reward setting. %However, this work establishes that bounded reward setting cannot be attacked by reward manipulation and action manipulation in episodic RL.  
Our results are compared with the relevant literature in  Table \ref{table:comparison_1}.% in   Appendix \ref{sec:comparison}. %\hf{Given our space, we may put this table back to main body. } \ar{done}%, and the attack cost is summation of amount of contamination and number of action manipulations by the attacker. 

\section{Problem Formulation}
We consider  the episodic Markov Decision Process (MDP), denoted by $(\mathcal{S},\mathcal{A}, H,\mathcal{P},\mu)$, where $\mathcal{S}$ is the set of states with $|\mathcal{S}|=S$, $\mathcal{A}$ is the set of actions with $|\mathcal{A}|=A$, $H$ is the number of steps in each episode, $\mathcal{P}$ is the transition metric such that $\mathbb{P}(.|s,a)$ gives the transition  distribution over the next state if action $a$ is taken in the current state $s$, and $\mu:\mathcal{S}\times\mathcal{A}\to \mathbb{R}$ is the expected reward of state action pair $(s,a)$. {To avoid cumbersome  notations, we work with the stationary MDPs here with the same reward  and transition functions at each $h\leq H$. However all our analysis extend trivially to non-stationary MDPs, with just more involved notations.\footnote{Moreover, if an online learning is already vulnerable for this ``nice'' special case, then it must be  vulnerable in general as well since the stationary MDP is a special case of general MDPs.} }

An RL agent (or learner) interacts with the MDP for $T$ episodes, and each episode consists of $H$ steps.
In each episode of the MDP, an initial state $s_t(1)$  can be fixed or selected from any distribution. In episode $t$ and step $h$, the learner observes the current state $s_t(h)\in\mathcal{S}$, selects an action $a_t(h)\in\mathcal{A}$, and incurs a noisy reward $r_{t,h}(s_t(h),a_t(h))$. % and observes a possibly corrupted (and thus altered)) reward $r^o_{t,h}(s_t(h),a_t(h))$ corresponding to the selected action.
Also, we have $\mathbb{E}[r_{t,h}(s_t(h),a_t(h))]=\mu(s_t(h),a_t(h))$. Our results can also be extended to the setting where reward is a function of step $h\leq H$. Finally, both $\mathcal{P}$ and $\mu$ are unknown to the learner and the attacker.

We consider episodic RL under MITM attacks.  The attacker can manipulate the action $a_t(h)$ selected by the learner to another action $a^o_t(h)\in \mathcal{A}$. The   MDP thus undergoes transition to next state based on the 
action  $a^o_t(h)$, namely the next state is drawn from the  distribution $\mathbb{P}(.|s_t(h), a^o_t(h))$.  The reward observation $r_t(s_t(h),a^o_t(h))$ is generated. If $a^o_t(h)\neq a_t(h)$, then the episode $t$ and step $h$ is said to be \emph{under action manipulation attack}. Hence, the \emph{number of action manipulations} is $\sum_{t=1}^T\sum_{h=1}^H\mathbf{1}(a^o_t(h)\neq a_t(h))$. 

The adversary can also intercept the realized reward   $r_t(s_t(h),a^o_t(h))$ and contaminate it by adding noise $\epsilon_{t,h}(s_t(h),a_t(h))$. Learner thus observes reward   %observed by the learner and the  reward $r_{t,h}(s_t(h),a^o_t(h))$ generated by the environment satisfy the following relation
\begin{equation}\label{eq:reward_observation}
    r^o_{t,h}(s_t(h),a_t(h))=r_{t,h}(s_t(h),a^o_t(h))+\epsilon_{t,h}(s_t(h),a_t(h)),
\end{equation}
where the contamination $\epsilon_{t,h}(s_t(h),a_t(h))$ added by the attacker is a function of the entire history, including all the states visited previously and all the actions selected previously by the learner and the attacker. %Since $r^o_{t,h}(s_t(h),a_t(h))\in [0,1]$, we have that $\epsilon_{t,h}(s_t(h),a_t(h))\in [-r_{t,h}(s_t(h),a_t(h)),1-r_{t,h}(s_t(h),a_t(h))]$. 
If $\epsilon_{t,h}(s_t(h),a_t(h))\neq 0$, then the episode $t$ and step $h$ is said to be \emph{under reward manipulation attack}. Hence, the \emph{number of reward manipulations} is $\sum_{t=1}^T\sum_{h=1}^H\mathbf{1}(\epsilon_{t,h}(s_t(h),a_t(h))\neq 0)$, and the \emph{amount of contamination} is 
$\sum_{t=1}^T\sum_{h=1}^H |\epsilon_{t,h}(s_t(h),a_t(h))|$. Notably,   \emph{reward manipulation attack} is a special case of MITM attack, where the adversary cannot manipulate the action (i.e., $a_t(h)=a^o_t(h)$). \emph{Action manipulation attack} is also special case of MITM, in which the adversary cannot contaminate the reward observation (i.e., $\epsilon_{t,h}(s_t(h),a_t(h))=0$).

A (deterministic) policy $\pi$ of an agent is a collection of $H$ functions $\{\pi_h:\mathcal{S}\to\mathcal{A}\}$. The value function $V_h^\pi(s)$ is the expected reward under policy $\pi$, starting from state $s$ at step $h$, until the end of the episode, namely
\begin{equation}\label{eq:ValueFunc}
    V_h^\pi(s)=\mathbb{E}\big[\sum_{h^\prime=h}^H \mu(s_{h^\prime},\pi_{h^\prime}(s_{h^\prime}))|s_h=s\big],
\end{equation}
where $s_{h^\prime}$ denotes the state at step $h^\prime$. Likewise, the $Q$-value function $Q_{h}^\pi(s,a)$ is the expected reward under policy $\pi$, starting from state $s$ and action $a$, until the end of the episode, namely
\begin{equation}\label{eq:Qfunnc}
\begin{split}
    Q^\pi_h(s,a)
    &=\mathbb{E}[\sum_{h^\prime=h+1}^H \mu(s_{h^\prime},\pi_{h^\prime}(s_{h^\prime}))|s_h=s, a_h=a]\\
    &\qquad +\mu(s,a),
\end{split}
\end{equation}
where $a_{h^\prime}$ denotes the action at step $h^\prime$. Since $S$, $A$ and $H$ are finite, there exists an optimal policy $\pi^*$ such that $V_h^{\pi^*}(s)=\sup_{\pi}V_h^{\pi}(s)$. 
The regret $R^\mathcal{A}(T,H)$ of any algorithm $\mathcal{A}$ is the difference between the total expected true reward from the best fixed policy $\pi^*$ in the hindsight, and the expected true reward over $T$ episodes, namely
\vskip -15pt
\begin{equation}
    R^{\mathcal{A}}(T,H)= \sum_{t=1}^T \big(V_1^{\pi^*}(s_t(1))-V_1^{\pi_t}(s_t(1))\big). 
\end{equation}
\vskip -3pt
The objective of the learner is to minimize the regret $R^{\mathcal{A}}(T,H)$. In contrast,  the objective of the attacker is to poison the environment to teach/force the learner to execute a target policy $\pi^+$ by achieving the following objective:  for all $h\leq H$, $s\in \mathcal{S}$ and any policy $\pi$, 
\vskip -8pt
\begin{equation}\label{eq:ObjectAttacker}
  \text{if $\pi_h(s)\neq \pi_h^+(s)$, then }  \tilde{V}_h^{\pi}(s)< \tilde{V}_h^{\pi^+}(s),
\end{equation}
\vskip -3pt
\noindent
 where $\tilde{V}_h^{\pi}(s)$ is the  expected  reward  in  states based on the  reward observation in \eqref{eq:reward_observation} under  policy $\pi$.
 Consequently, the policy $\pi^+$ will be executed for $\Omega(T)$ times under the attack. % \hf{Then what is $\pi$ in (6) then? is the above rephrase   accurate? }  
 
% \hf{I am not sure whether this objective is rigorously correct (I know it is intuitively true) --- it is possible that the learner sometimes follow $\pi$ anyway even if it is not optimal, due to exploration? So the following suffice for executing a target policy $\pi^+$ at least $\Omega(T)$ times but is not necessary, right? }\ar{I think $\Omega(T)$ is just a consequence of (6). However, in our theorems, we explicitly show following equation holds.}

%\ar{Note that the strict inequality is crucial for saying the action manipulation attack alone does not work in bounded reward setting.}
 
\section{Attacks in Bounded Reward Setting}
In this section, we investigate bounded reward setting, i.e.,  $r_{t,h}(s,a)\in [0,1]$ with mean $\mu(s,a)\in (0,1]$ for all $(s,a)\in \mathcal{S}\times\mathcal{A}$.\footnote{The $\mu(s,a) \not = 0$ is to avoid   minor technicality issues due to tie breaking. } 
%\hf{is there a reason that $\mu(s,a)\in (0,1]$ but not $\mu(s,a)\in [0,1]$?}\ar{we are pulling down everything to 0, and we would like that case to be sub-optimal for all the target policies due to strict inequality in our objective function}. 
We first show in subsection \ref{sec:bounded:insufficiency} that there exist MDPs and target policies $\pi^+$ such that  the objective of the attacker, namely \eqref{eq:ObjectAttacker}, cannot be achieved by only reward manipulation attack or only action manipulation attack.  In subsection \ref{sec:bounded:combined}, we show that combined reward and action manipulation suffice for a successful attack.  

\subsection{Insufficiency of (Only) Reward or Action Manipulation} \label{sec:bounded:insufficiency}
Similar to \cite{rakhsha2020policy}, attacker is subject to a constraint that the reward manipulation (i.e.,  $\epsilon_{t,h}(s_t(h),a_t(h))\neq 0$) and action manipulation (i.e.,  $a_t(h)\neq a^o_t(h)$) can occur only if the 
selected action  is different from the desired action, namely $a_t(h)\neq \pi^+_h(s_t(h))$.Intuitively, it can also be interpreted as the attacker can manipulate rewards (or actions) whenever the learner's action is not in accordance with the attacker's desired policy $\pi^+$.
Therefore, in the reward manipulation attack, the attacker is subject to following constraints
\vskip -10pt
\begin{equation}\label{eq:rewardManiConst}
\begin{split}
     &r^o_t(s_t(h),a_t(h))= r_t(s_t(h),a_t(h)) \mbox{ if } a_t(h)= \pi^+_h(s_t(h)),\\
     &\mbox{ and } r^o_t(s_t(h),a_t(h))\in [0,1],
\end{split}
\end{equation}
\vskip -15pt
or equivalently, 
\begin{equation}
\begin{split}
     &\epsilon_t(s_t(h),a_t(h))= 0 \mbox{ if } a_t(h)= \pi^+_h(s_t(h)),  \mbox{ and }\\
    & \epsilon_t(s_t(h),a_t(h))\in [-r_t(s_t(h),a_t(h)),1-r_t(s_t(h),a_t(h))].
\end{split}
\end{equation}
\vskip -5pt
This constraint is crucial for obtaining sub-linear attack cost due to the following reason. Suppose the objective in \eqref{eq:ObjectAttacker} is achieved by contaminating the reward of action $\pi^+_h(s_t(h))$. Now since the learner would execute the policy $\pi^+$ with high probability, namely $\Omega(T)$ times,  the total contamination will thus grow linearly with  $T$.  This constraint (or strategy of not contamination $\pi^+_h(s_t(h))$) is also applied in the previous literature in RL  \cite{rakhsha2020policy} and MAB \cite{jun2018adversarial,ma2019data,rangi2022saving}, where the reward manipulation are performed only on non-desirable actions. %to achieve sub-linear attack cost. 
%\ar{@Haifeng, explanation of the constraint above. Did not talk about positive and negative bias, and how they can vanish if we stop. } \hf{This looks great.}
%
Similarly, in the action manipulation attack, we assume that %the attacker is subject to following constraints
\vskip -8pt
\begin{equation}\label{eq:actionManiConst}
    a^o_t(h)=  a_t(h)\mbox{ if } a_t(h)= \pi^+_h(s_t(h)).
\end{equation}
\vskip -3pt
That is, the action can be manipulated only if the selected action is not the same as the target action of the policy. 
Our first result establishes that only the reward manipulation or only the action manipulation cannot always guarantee successful attacks in bounded reward setting.  
\begin{theorem}\label{thm:infeasibleRewardMani}
In bounded reward setting,  
\begin{enumerate}
    \item there exists an MDP and a target policy $\pi^+$ such that any reward manipulation attack satisfying \eqref{eq:rewardManiConst} cannot be successful, namely achieve the objective in \eqref{eq:ObjectAttacker}.
    \item there exists an MDP and a target policy $\pi^+$ such that any action manipulation attack satisfying \eqref{eq:actionManiConst} cannot be successful, namely achieve the objective in \eqref{eq:ObjectAttacker}.
\end{enumerate}
\end{theorem}
 
% \hf{This impossibility results hold even when the attack cost can be $\Omega(T)$, right? We did not state clearly at the preliminary part that whether we only ocnsider attack with $o(T)$ cost or we allow $\Omega(T)$ costs.  }\ar{No.  The condition is that manipulation can happen only if the action is not same as the targeted action. It does not allow $\Omega(T)$ manipulations. Number of manipulations allowed will be dependent on the regret of the algorithm. I agree that explicit conditions are not mentioned in the contribution. }

 The proof of Theorem \ref{thm:infeasibleRewardMani} proceeds by constructing a MDP such that there exists a  target policy $\pi^+$ and another policy $\tilde{\pi}$  such that $\tilde{V}_h^{\tilde{\pi}}(s)>\tilde{V}_h^{{\pi^+}}(s)$ irrespective of action or reward manipulation. This implies that the constructed target policy $\pi^+$ can not be induced in the constructed MDP.\footnote{Note that such  \emph{worst-case} analysis is typical for establishing lower bounds on regret, which identifies difficult instance and shows the impossibility
 of a good regret (e.g., \cite{bogunovic2020stochastic}). }

\subsection{Efficient Attack by Combining Reward \& Action Manipulation}\label{sec:bounded:combined}
We now show that the attacker can achieve its objective by combining the strength of both reward manipulation and action manipulation attacks in bounded reward setting. %In other words, action manipulation is not required, namely $a^o_t(h)=a_t(h)$. 
 This is the first attack strategy of its kind, using both reward and action manipulation, proposed in the literature. 
%Additionally, the attack cost, namely sum of the amount of contamination and the number of action manipulation, is $\tilde{\Theta}(\sqrt{T})$. 

% Here we start directly with the black-box attack here. For curious readers, we provide a discussion about the less realistic yet simpler  white-box attack in  Appendix \ref{append:white-bounded}. 
Given the target policy $\pi^+$, for all $s_t(h)\in\mathcal{S}$, $a_t(h)\in\mathcal{A}$ and $h\leq H$, we consider the following attack strategy  
\vskip -5pt
\begin{equation}\label{eq:trasitionCorruption_unin}
    a^o_t(h)=\begin{cases} a_t(h)&\mbox{ if } a_t(h)= \pi_h^+(s),\\
      \pi_h^+(s) &\mbox{ if }  a_t(h)\neq \pi_h^+(s),
    \end{cases}
\end{equation}
%and
\vskip -5pt
\begin{equation}\label{eq:RL_fullInfo2_uninb}
    r_t^o(s_t(h), a_t(h))=\begin{cases} r_t(s_t(h),a_t(h))&\mbox{ if } a_t(h)= \pi_h^+(s),\\
      0 &\mbox{ if } a_t(h)\neq \pi_h^+(s).
    \end{cases}
\end{equation}
In the above attack, the adversary manipulates both the action and the reward observation when $a_t(h)\neq \pi^+_h(s)$. Specifically, the adversary manipulates the action to $\pi^+_h(s)$ to control the transition dynamics, and at the same time  manipulates the reward observation to zero so that the action $a_t(h)$ appears to be sub-optimal in comparison to the action  $\pi^+_h(s)$.{In this attack,
we carefully “coordinate” action and reward manipulation to achieve two goals simultaneously: $(1)$ the target policy is
 optimal; $(2)$ the target policy is the greedy solution of the resultant MDP, namely the action yielding maximum
reward at current step $h$ is also the action yielding maximum cumulative reward between 
$h$ and $H$. }%In advance, it  is not even clear whether these can be achieved.}

%In other words, the attacker uses its power to manipulate both rewards and actions whenever the learner's action is not in accordance with the attacker's desired policy $\pi^+$. 

%As shown in Theorem \ref{thm:infeasibleRewardMani}, such manipulation of both rewards and actions is necessary. %  since  the MDP environment cannot be always attacked by the reward manipulation or action manipulation alone.
The following theorem shows that %once the strength of reward manipulation and action manipulation are available together, 
 the novel attack in \eqref{eq:trasitionCorruption_unin} and \eqref{eq:RL_fullInfo2_uninb} can achieve the objective in \eqref{eq:ObjectAttacker}, and does not require learning the parameters of the MDP. Additionally, the attack cost is $\tilde{O}(\sqrt{T})$ for any order-optimal learning algorithm.
\begin{theorem}\label{thm:boundedBlackBox}
Consider any learning algorithm $\mathcal{A}$ such that its regret in the absence of attack is 
\vskip -5pt
\begin{equation}\label{eq:regA_13}
    R^{\mathcal{A}}(T,H)=\tilde O(\sqrt{T}H^{\alpha}) \qquad \forall T\geq t_0, 
\end{equation}
with probability at least $1-\delta$, where $\alpha\geq 1$ is a numerical constant. For any sub-optimal target policy $\pi^+$, if an adversary follows the   strategy in \eqref{eq:trasitionCorruption_unin} and \eqref{eq:RL_fullInfo2_uninb} to attack the   algorithm $\mathcal{A}$,  then with probability at least $1-\delta$,  the following statements hold simultaneously:
\begin{enumerate}
    \item  The attacker achieves its objective in \eqref{eq:ObjectAttacker} and moreover $\sum_{t=1}^T\sum_{h=1}^H\mathbf{1}(a_t(h)= \pi_h^+(s_{t}(h)))=\Omega(T);$ 
    \item The number of reward manipulation attacks, namely $\sum_{t=1}^T\sum_{h=1}^H\mathbf{1}(\epsilon_{t,h}(s_t(h),a_t(h))\neq 0)$ is $\tilde O\big({\sqrt{T}H^\alpha}/\min_{h,s}\mu(s,\pi^+_h(s))\big)$ %upper bounded by 
% \begin{equation}\label{eq:NumFinalRes_1_1_2}
 %\begin{split}
 %     &\sum_{t=1}^T\sum_{h=1}^H\mathbf{1}(\epsilon_{t,h}(s_t(h),a_t(h))\neq 0)\\
 %     &=\tilde O\big({\sqrt{T}H^\alpha}/\min_{h,s}\mu(s,\pi^+_h(s))\big),
 %\end{split}
 %\end{equation}
 \item The amount of reward contamination, namely $\sum_{t=1}^T\sum_{h=1}^H|\epsilon_{t,h}(s_t(h),a_t(h))|$, is $\tilde O\big({\sqrt{T}H^\alpha}/\min_{h,s}\mu(s,\pi^+_h(s))\big)$ %upper bounded by 
 %\begin{equation}\label{eq:NumFinalRes_1_1_2_1}
 %\begin{split}
 %     &\sum_{t=1}^T\sum_{h=1}^H|\epsilon_{t,h}(s_t(h),a_t(h))|\\
 %     &=\tilde O\big({\sqrt{T}H^\alpha}/\min_{h,s}\mu(s,\pi^+_h(s))\big),
 %\end{split}
 % \end{equation}
 \item The number of action manipulation attacks, namely $\sum_{t=1}^{T}\sum_{h=1}^H\mathbf{1}(a^o_t(h)\neq a_t(h))$, is $\tilde O\big({\sqrt{T}H^\alpha}/{\min_{h,s}\mu(s,\pi^+_h(s))}\big)$.
%\begin{equation}\label{eq:actionManiAttack}
%\begin{split}
%        &\sum_{t=1}^{T}\sum_{h=1}^H\mathbf{1}(a^o_t(h)\neq a_t(h))\\
%        &=\tilde O\big({\sqrt{T}H^\alpha}/{\min_{h,s}\mu(s,\pi^+_h(s))}\big).
%\end{split}
%\end{equation}
\end{enumerate}
\end{theorem} 
%{For obtaining the theorem, we first show that the target policy is optimal for the learner based on the observed reward and state of the MDP, more precisely
%\begin{equation}
%\begin{split}
%      \bar{Q}^\pi_{h}(s,a)
%     <\bar{Q}^{\pi^+}_{h}(s, \pi_{h}^+(s)),
%\end{split}\end{equation}
%where $\bar{Q}^\pi_h(s,a)$ is the expected reward in state $s$ for action $a$ introduced by the above reward and action manipulation under policy $\pi$. 
%Combining this fact along with the regret guarantees, we prove \eqref{eq:NumFinalRes_1_1_2}, and remaining results are derived thereon.}

{In \cite{chen2021improved}, if the attacker can observe learner's action (as in our setting), BARBAR-RL has a regret $\tilde{O}(\sqrt{T}+C^2)$ where $C$ is the total contamination in rewards and transition dynamics. Since \emph{action and reward manipulation} is a special case of \emph{transition dynamics and reward manipulation}, the bound of BARBAR-RL will also hold in our setting. If attacker is absent ($C=0)$, then the regret of the algorithm is $\tilde{O}(\sqrt{T})$, and \eqref{eq:regA_13} holds. If $C=o(\sqrt{T})$, then the regret of BARBAR-RL is $o(T)$, and claim $1$ in Theorem \ref{thm:boundedBlackBox} does not hold. Thus, there exists an order-optimal algorithm such that attacker's objective is not achieved in $o(\sqrt{T})$ attack cost. Hence, the attack cost of proposed strategy is order-optimal.}

The order-optimal algorithm, in \eqref{eq:regA_13}, has three parameters $\alpha$, $t_0$ and $\delta$. The parameter $\alpha$ is used to capture the dependence of regret on $H$; parameter $t_0$ captures the fact that the regret bound holds for sufficiently large $T$; parameter $\delta$ captures the fact that the regret bound holds with high probability. Some examples of order-optimal algorithm in episodic RL can be found in \cite{jin2018q} and references therein. We present results on attack cost for order-optimal algorithms, however this is not a limitation of our attack strategy. The analysis can be extended for any no-regret learning algorithm and a corresponding bound on attack cost can be obtained. 
%Notably, Equation \eqref{eq:NumFinalRes_1_1_2} also implies that the total amount of reward manipulation is also upper bounded by $O\big({\sqrt{T}H^\alpha}/\min_{h,s}\mu(s,\pi^+_h(s))\big)$ because   $|\epsilon_{t,h}(s_t(h),a_t(h)| \leq \mathbf{1}(\epsilon_{t,h}(s_t(h),a_t(h) \not = 0)  $ always hold due to bounded rewards.   

Our result directly focus on the more realistic and also harder black-box attacks; for completeness, we also discuss how to design a white-box attack by manipulating   rewards and actions in the Appendix. %\ref{append:white-bounded}. 
This can also be used for designing a black-box attack in unbounded reward setting.
 
% Comparing Theorem \ref{thm:boundedBlackBox} %and our white-box attack in Appendix \ref{append:white-bounded}, we can conclude that the attack cost of both black-box and white-box attack in this setting is $O(\sqrt{T})$.  Unlike the unbounded reward setting, in bounded reward setting, the attack cost of the white-box and black box attack only vary by a numerical constant. This is due to the fact that amount of contamination in each round is bounded above by unity. 

\section{Reward Poisoning in Unbounded Settings}\label{sec;unreward}
In this section, we investigate the unbounded reward setting. Formally,  for all $(s,a)\in \mathcal{S}\times\mathcal{A}$, we assume $r_{t,h}(s,a)$ follows a sub-Gaussian distribution  with mean $\mu(s,a)$ and standard deviation $\sigma$. While the rewards, drawn from sub-Gaussians, may be unbounded, we assume their mean $\mu(s,a) \in [-M, M]$ is bounded within some known interval for any $(s,a)$ pair. 
% It turns out that in this case  the attacker can achieve its objective with just   reward manipulation only if single reward observation $\{r_0(s,a)\}_{s,a}$ corresponding to each state action pair is known. 
{Our main result in this section is the design of a black-box attack to \emph{any} order-optimal episodic RL algorithm with  $\tilde{O}(\sqrt{T})$ amount of reward contamination.  To our knowledge, this is the first efficient attack to episodic RL algorithms in unbounded reward settings. } 
% In other words, action manipulation is not needed. Additionally, the amount of contamination needed is $\tilde{O}(\sqrt{T})$. 
\begin{algorithm}[b!]
\algsetup{linenosize=\small}
\small
\begin{algorithmic}[1]
% \STATE \textbf{Input:} Initial observations $\{r_0(s,a)\}_{s,a}$
%\STATE \textbf{Initialization:} For all $(s,a,h)\in \mathcal{S}\times \mathcal{A}\times [H]$, attacker initializes $\hat{\mu}(s,a)=r_0(s,a)$, $\hat{\mu}^{UCB}(s,a)=r_0(s,a)+\sqrt{4\log(2THSA)}$, $\hat{\mu}^{LCB}(s,a)=r_0(s,a)-\sqrt{4\log(2THSA)}$  and $N(s,a)=1$.
\STATE \textbf{Initialization:} Input parameters are $\epsilon>0$ and policy $\pi^+$. For all $(s,a,h)\in \mathcal{S}\times \mathcal{A}\times [H]$, attacker initializes $\hat{\mu}(s,a)=0$, $\hat{\mu}^{UCB}(s,a)=M$, $\hat{\mu}^{LCB}(s,a)=-M$  and $N(s,a)=0$. %\hf{note my change to $N(s,a)=0$ now.}
\FOR{episode $t\leq T$}
\STATE Observe the initial state $s_t(1)$
\FOR{  $h\leq H$}
\STATE Observe the selected action $a_t(h)$, reward  $r(s_{t}(h),a_t(h))$ and next state $s_{t}(h+1)$. 
\STATE Update 
\vskip -12pt
\begin{equation}\label{eq:mu1}
\begin{split}
      &\hat{\mu}(s_t(h),a_t(h))\\
      &=\frac{\hat{\mu}(s_{t}(h),a_t(h))N(s_{t}(h),a_t(h))+ r(s_{t}(h),a_t(h))}{N(s_{t}(h),a_t(h))+1}, 
\end{split}
\end{equation}
\vskip -15pt
\begin{equation}\label{eq:mu1UCB}
\begin{split}
     &\hat{\mu}^{UCB}(s_{t}(h),a_t(h))\\&=\hat{\mu}(s_t(h),a_t(h))+\sigma\sqrt{\frac{4\log(2THSA)}{N(s_{t}(h),a_t(h))+1}},  
\end{split}
\end{equation}
\vskip -12pt
\begin{equation}\label{eq:mu1LCB}
\begin{split}
     &\hat{\mu}^{LCB}(s_{t}(h),a_t(h))\\
     &=\hat{\mu}(s_t(h),a_t(h))-\sigma\sqrt{\frac{4\log(2THSA)}{N(s_{t}(h),a_t(h))+1}}, 
\end{split}
\end{equation}
and $N(s_{t}(h),a_t(h))=N(s_{t}(h),a_t(h))+1$.
%\IF{$\exists (s,a)\in \mathcal{S}\times \mathcal{A},$ such that $N(s,a)=0$}
%\STATE Contaminate the reward observation such that $r^o_t(s_{t}(h),a_t(h))=\mbox{Bern}(1/2)$.
%\ELSE
\IF{$a_t(h)=\pi_h^+(s_t(h))$}
\STATE Do not contaminate. %, namely $r_t^o(s_t(h), a_t(h))=r_t(s_t(h),a_t(h))$. 
\ELSE
\STATE Contaminate the reward observation such that
\vskip -12pt
\begin{equation}\label{eq:rewardCont1}
\begin{split}
    & r_t^o(s_t(h), a_t(h)) \\
    &=  \hat{\mu}^{LCB}(s_{t}(h),\pi_h^+(s_t(h)))-\epsilon\\
    &+(H-h) \min_{s,a\in \mathcal{S}\times \mathcal{A}}\hat{\mu}^{LCB}(s,a) \\
    & - (H-h)  \max_{s,a\in \mathcal{S}\times \mathcal{A}}\hat{\mu}^{UCB}(s,a)
      .
\end{split}
\end{equation}
\vskip -10pt
\ENDIF
%\ENDIF
\ENDFOR
\ENDFOR
\caption{Black box attack strategy}
\label{alg:attackRLUninnf}
\end{algorithmic}
\end{algorithm}

Formally, we consider the black-box attack setting in which the attacker doesn't know about the  expected reward and the transition dynamics of the underlying MDP. %However, unlike the black-box setting, the attacker has initial observations from the reward distribution. This information can be collected by the attacker by limited eavesdropping on the environment or by interacting with the environment as a learner itself. Note that the initial information consists of $\mathcal{S}\mathcal{A}$ observations, which is a constant, and does not scale with $H$ and $T$. Hence, the cost related to collecting this information is constant. 
In this setting, we propose an attack which learns about the MDP, and has almost the same attack cost as the white-box attack, with an additional $O(\sqrt{\log T})$ factor.
Given the input parameters $\pi^+$ and $\epsilon>0$, our proposed attack strategy is presented in Algorithm \ref{alg:attackRLUninnf}. %Unlike the white box attack, the attack in the black box setting evolves in two phases: initialization phase and exploitation phase. In the initialization phase, the objective of the attacker is to obtain at least one observation of the reward from the environment corresponding to every $(s,a)\in \mathcal{S}\times\mathcal{A}$ pair. This helps in initializing the estimate $\hat{\mu}(s,a)$ of the true mean $\mu(s,a)$ for each $(s,a)\in \mathcal{S}\times\mathcal{A}$ pair. In Algorithm \ref{alg:attackRLUninnf}, the attacker achieves this by contaminating the observed reward $r_t^o(s,a)\sim \mbox{Bern}(0.5)$ for each $(s,a)$ pair, where $ \mbox{Bern}(p)$ denotes the Bernoulli distribution with mean $p$. This contamination ensures that the expected reward of each $(s,a)$ pair selected by the the learner appears identical, and thus promotes the exploration of all the $(s,a)$ pairs. 
%The initialization phase stops if the reward corresponding to each $(s,a)$ pair is observed at least once, namely  $N(s,a)\geq 1, \forall (s,a)\in \mathcal{S}\times\mathcal{A}$  in Algorithm \ref{alg:attackRLUninnf}, where $N(s,a)$ denotes the number of times the $(s,a)$ pair is selected. 
In Algorithm \ref{alg:attackRLUninnf}, the attacker utilizes its estimate of $\hat{\mu}(s,a)$, $\hat{\mu}^{UCB}(s,a)$ and $\hat{\mu}^{LCB}(s,a)$, defined in \eqref{eq:mu1}, \eqref{eq:mu1UCB} and \eqref{eq:mu1LCB} respectively,  to contaminate the reward observations. These estimates are initialized using the fact that $\mu(s,a)\in [-M,M]$, and are updated in each episode $t\leq T$ and at each step $h\leq H$. The parameters $\hat{\mu}^{UCB}(s,a)$ and $\hat{\mu}^{LCB}(s,a)$ are Upper Confidence Bound (UCB) and Lower Confidence Bound (LCB) of $\mu(s,a)$. Therefore, we  show that,  with high probability,
\vskip -3pt
\begin{equation}\label{eq:rewardObjetiveCon}
    \hat{\mu}^{LCB}(s,a)\leq \mu(s,a)\leq \hat{\mu}^{UCB}(s,a).
\end{equation}
In this attack, the reward observations are contaminated only if the action selected by the learner is not the same as the action desired by the target policy, namely $a_t(h)\neq \pi^+_h(s_t(h))$. In this scenario, the reward observation $r^o_t(s_t(h),a_t(h))$ is defined in \eqref{eq:rewardCont1}. In \eqref{eq:rewardCont1}, LCB estimates of the expected rewards are used to get the LCB estimate of the value function of \textit{target policy} and UCB estimates of the expected rewards are used to get UCB estimate of the value function over \textit{all policy}. The reward observation consists of a large negative bias sufficient to achieve the objective. This reward manipulation strategy in \eqref{eq:rewardCont1} ensures that the target policy $\pi^+$ is the optimal policy based on the observed reward observations, namely for all $h\leq H$, and $(s,a)\in \mathcal{S}\times\mathcal{A}$ such that $a\neq \pi_h^+(s)$, we have
\begin{equation}\label{eq:toprove}
    \tilde{Q}^{\pi}_{h}(s,a)\leq \tilde{Q}_{h}^{\pi+}(s,\pi^+_h(s))-\epsilon,
\end{equation}
where $\tilde{Q}^{\pi}_h(s,a)$ is the expected reward in state $s$ for action $a$ for the  reward observation under policy $\pi$. %and $\tilde{V}^{\pi}_h(s)$ is the expected reward in state $s$ for the above reward observation under policy $\pi$. 
These values will not be the same as the ones defined in \eqref{eq:ValueFunc} and \eqref{eq:Qfunnc} since the reward observations are manipulated. We remark that the $r_t^o(s_t(h), a_t(h))$ can be computed through a backward induction procedure starting from horizon $H$. At any step $h$ in the episode, the definition of $r_t^o(s_t(h), a_t(h))$ depends linearly on the Q-values at $h$, which then depends linearly on $r_t^o(s_t(h), a_t(h))$. Therefore, $r_t^o(s_t(h), a_t(h))$ at any horizon $h$ can be computed by solving a linear system. 

We briefly discuss the key steps in this process of obtaining \eqref{eq:toprove}. 
 We show that with high probability 
 \vskip -10pt
\begin{equation}\label{eq:param1}
\begin{split}
    &\hat{\mu}^{LCB}(s_{t}(h),\pi_h^+(s_t(h)))+(H-h) \min_{s,a\in \mathcal{S}\times \mathcal{A}}\hat{\mu}^{LCB}(s,a)\\
    &\leq \tilde{Q}^{\pi^+}_h(s_t(h), \pi_h^+(s_t(h))).
\end{split}
\end{equation}
\vskip -10pt
Additionally, we have that with high probability 
 \vskip -10pt
\begin{equation}\label{eq:param2}
\begin{split}
     &(H-h)  \max_{s,a\in \mathcal{S}\times \mathcal{A}}\hat{\mu}^{UCB}(s,a) \\
     &\geq \mathbb{E}_{s^\prime\sim \mathbb{P}(s^\prime|s_{t}(h),a_t(h))}[\tilde V_{h+1}^{\pi^+}(s^\prime)].
\end{split}
\end{equation}
Combining \eqref{eq:param1} and \eqref{eq:param2}, we have that with high probability, the rewards contamination in Algorithm \ref{alg:attackRLUninnf} ensures \eqref{eq:rewardObjetiveCon}.

The following theorem shows that our proposed   black box attack has $\tilde{O}(\sqrt{T})$ amount of contamination.  
\begin{theorem}\label{thm:BBAtackun}
Consider any learning algorithm $\mathcal{A}$ such that its regret in the absence of attack is $ R^{\mathcal{A}}(T,H)=\tilde O(\sqrt{T}H^{\alpha}), \qquad \forall T\geq t_0$
\begin{equation}\label{eq:regA}
    R^{\mathcal{A}}(T,H)=\tilde O(\sqrt{T}H^{\alpha}), \qquad \forall T\geq t_0
\end{equation}
with probability at least $1-\delta$  where $\alpha\geq 1$ is a numerical constant.  
For any sub-optimal target policy $\pi^+$, $\epsilon>0$ and $T\geq t_0^2$,
%and for all $(s,a)\in S\times A$, we have that state action pair $(s,a)$ is selected at least once in $O(\sqrt{T})$ episodes, namely
%\begin{equation}
%    N_{\sqrt{T}}^H(s,a)\geq 1.
%\end{equation}
 if an attacker follows strategy in Algorithm \ref{alg:attackRLUninnf}, then with probability at least $1-\delta-2/(HSAT)$ the following hold simultaneously: 
 \begin{enumerate}
     \item The attacker achieves its objective in  \eqref{eq:ObjectAttacker} and moreover $\sum_{t=1}^T\sum_{h=1}^H\mathbf{1}(a_t(h)= \pi_h^+(s_{t}(h)))=\Omega(T)$; 
     \item The number of reward manipulations is $\tilde O\big({\sqrt{T}H^\alpha}/{\epsilon}\big)$.
 %\begin{equation}\label{eq:NumFinalRes}
 %\begin{split}
 %     &\sum_{t=1}^T\sum_{h=1}^H\mathbf{1}(\epsilon_{t,h}(s_t(h),a_t(h))\neq 0)= \tilde O\big({\sqrt{T}H^\alpha}/{\epsilon}\big);
 %\end{split}
 %\end{equation}
 \item The total amount of reward contamination is $\tilde O\big(\sqrt{T}H^{\alpha+1}(\epsilon+\sqrt{\log (HTSA)} )/\epsilon\big)$.
%\begin{equation}\label{eq:conntaminBB}
%\begin{split}
%    &\sum_{t=1}^T\sum_{h=1}^H|\epsilon_{t,h}(s_t(h),a_t(h)|\\
%    &= \tilde O\big(\sqrt{T}H^{\alpha+1}(\epsilon+\sqrt{\log (HTSA)} )/\epsilon\big),
%\end{split}
%\end{equation}
 \end{enumerate}
\end{theorem}

Additionally, the proposed black-box attack has an additional cost $O(\sqrt{\log T})$ in comparison to the white-box attack in unbounded reward setting.%(refer appendix). %(see Theorem \ref{thm:unBoundedWhiteBox} in the appendix).

%\hf{Again, I delete the white box attack description, but for this unboudned reward case, we may add it back later if we have the space since they seem to be interesting. }

%\ar{@Haifeng, please check the following comparison.} \hf{rephrased, please check.}
\begin{remark} \label{rem;1} In unbounded setting, reward poisoning attack  has been studied in black-box setting recently in RL for infinite horizon by \cite{rakhsha2021reward}. They consider attacking $L$ online learners whereas we only has one learner. The  attack objective of \cite{rakhsha2021reward}  is to force all these  learning algorithms to execute the target policy $\pi^+$.  
We now highlight the key differences between our attack strategy  and the strategy of \cite{rakhsha2021reward}. The attack cost in \cite{rakhsha2021reward} is $\tilde{O}(T\log L)$, which is linear in $T$. On contrary, the attack cost of our proposed attack is $\tilde{O}(\sqrt{T})$. This difference occurs because the attack in \cite{rakhsha2021reward} estimates both the expected reward and the transition dynamics of the MDP. This is done by an explore-then-exploit form of strategy which leads to $\tilde{O}(T\log L )$ attack cost. 
%Their attacker explores until the parameters of the MDP are estimated with precision, and that exploration phase last $O(T\log L)$ rounds, which is linear in $T$. On contrary,  our exploration phase, which is the initialization  phase, is completed within $O(\sqrt{T})$ episodes since it requires only a single observation for each $(s,a)$ pair, and does not wait for the precise estimate of the parameters associated with the MDP. 
On contrary, our strategy focuses on estimating expected rewards only (and not transition dynamics). It
 compensates for this lack of knowledge of transition dynamics by adding a negative bias $O(\sqrt{\log T})$ to the reward observation. However, this additional cost is minimal in comparison to the cost of learning transition dynamics in \cite{rakhsha2021reward}. {This 
key technical insight allows us to reduce the attack cost of $O(T)$ in \cite{rakhsha2021reward} to $\tilde{O}(\sqrt
 T)$ in current work. }%Note that availability of additional information regarding rewards will not improve the attack cost in \cite{rakhsha2021reward}.
% \color{red}{ TODO: Add why learning the distribution is important in Rakhsha setting. Because of infinite horizon. TBD}
 
 %Thus, our attack strategy saves the cost of learning in exploration phase. Additionally, unlike \cite{rakhsha2021reward}, our attack strategy does not attempt to learn the transition dynamics, which reduces the number of learning parameters. All these together makes our attack more effective. Indeed, the attack cost of the proposed attack in \cite{rakhsha2021reward} is $\tilde{O}(T\log L+L\sqrt{T})$. Finally, if our attack strategy is applied to $L$ learners,  the attack cost is $\tilde{O}(L\sqrt{T})$, which is better by an additive factor $O(T\log L)$. This is a significant save when $T \gg L $ (i.e., $L = o(T)$).  %This is due to the fact that the attack in \cite{rakhsha2021reward} learns about the MDP using $O(\log L)$ learners for their entire duration $T$, thus, leading to a long exploration phase and higher attack cost. 
\end{remark}

\section{Additional Related Work}
Reward manipulation attack has been studied extensively in  MAB \cite{jun2018adversarial,liu2019data,rangi2022saving}, where the attacker's objective is to mislead the learner to choose a suboptimal action. %It has been shown that the optimal amount of contamination in unbounded reward setting is $\Theta(\sqrt{\log T})$\cite{zuo2020near}, and in bounded reward setting is $\Theta(\log T)$ \cite{rangi2022saving}, implying that bounded reward setting is more difficult to attack in terms of cost.
Action manipulation attack has also been studied in MAB \cite{liu2020action}, and the number of action manipulations required by the attacker is $O(\log T)$. All these attacks are studied in a Black-box setting. % where the attacker does not posses any knowledge about the underlying reward distributions. 
In this work, we show that unlike MAB setting, reward manipulation (only) and action manipulation (only) are not sufficient to successfully attack Episodic RL setting with bounded rewards.
 \cite{bogunovic2020stochastic} considered the ``possibility'' of reward poisoning attack  in Linear Bandits, as a step towards showing lower bound for designing robust algorithms. While their results can be extended to prove a similar regret lower bound  for RL, this only means that there \emph{exists} an RL instance   such that the attacker can successfully attack any no-regret algorithm for this particular instance.  %The work also claims that similar result can be obtained for RL. 
However, this  existence result   is significantly different from our perspective of designing attacks, in which we design      strategies that can successfully attack an \emph{arbitrary} episodic RL instance. % is the focus of the current work is not limited to an effective attack on a single RL instance. We propose an effective attack strategy for any episodic RL instance. }

In online RL setting, studies related to poisoning attacks have only   started recently, and have primarily focused  on  white-box settings, where the attacker has complete knowledge of   the underlying MDP models, with unbounded rewards \cite{rakhsha2020policy}.
In such white-box attacks, \cite{rakhsha2020policy} show that   reward poisoning attack requires $\tilde\Theta(\sqrt{T})$ amount of contamination to attack any order-optimal learning algorithm; they also show that    dynamic manipulation attack can achieve the same success with similar amount of cost in unbounded reward setting under some sufficient conditions.  \cite{zhang2020adaptive}  study  the feasibility of the reward poisoning attack in white box setting for $Q$-learning, and the attacker is constrained by the amount of contamination.  In a slightly different thread,  \cite{huang2020manipulating} analyse the  degradation of the performance of Temporal difference learning and Q-learning under falsified rewards.  %In this work, we propose black-box attack

%\ar{@Haifeng, can you please check the following paragraph  specifically?} hf{rephrased, please check}\\
To our knowledge, \cite{rakhsha2021reward} is the only work that studies   poisoning attack   in black-box setting for policy teaching in \emph{infinite-horizon} RL. However, they focused on the  settings with $L$ online learners, and the objective of their attacker is to force all these  learners to execute a target policy $\pi^+$. They proposed an attack with   $\tilde{O}(T\log L+L\sqrt{T})$ amount of contamination when $L$ is large enough.  However, our work focuses on attacking a \emph{single} learner and thus our setting is not comparable to \cite{rakhsha2021reward}. However, we can indeed apply our attack repeatedly to different learners to obtain an effective attack strategy for the setup of \cite{rakhsha2021reward}, which leads to an attack cost of $\tilde{O}(L\sqrt{T})$ (note however, the attack of \cite{rakhsha2021reward} cannot work for small $L$, e.g., $L=1$ as in our setup) in \emph{episodic} RL. This improves their attack cost by an additive amount $O(T\log L)$. Our more efficient attack is due to a more efficient design for the adversary to explore and learn the MDP, which is discussed at length in  Remark \ref{rem;1} in Section \ref{sec;unreward} . 
%the fact that the attack in \cite{rakhsha2021reward} learns about the MDP using $O(\log L)$ learners for their entire duration $T$, thus, leading to a long exploration phase and higher attack cost. 

 \emph{Test-time} adversarial attacks against RL has also been studied. Here, however, the policy $\pi$ of the RL agent is pre-trained and fixed, and the objective of the attacker is to manipulate the perceived state of the RL agent in order to induce undesired action \cite{huang2017adversarial,lin2017tactics,kos2017delving,behzadan2017vulnerability}. 
Such test-time attacks do not modify the the policy $\pi$, whereas  training-time attacks we study in this paper aims at poisoning the learned policy directly and thus may have a longer-term bad effects.  There have also been studies on reward poisoning against  \emph{Batch RL}  \cite{ma2019policy,zhang2009policy} where the attacker can modify the pre-collected batch data set at once. The focus of the present work is on \emph{online} attack where the poisoning is done on the fly.

\section{Conclusion}
This paper tries to understand poisoning attacks in RL. Towards that end,  we propose a reward manipulation attack for unbounded reward setting which successfully fool any order-optimal RL algorithm to pull a target policy with $\tilde{O}(\sqrt{T})$ attack cost. Extending the study to bounded reward setting, we show that the adversary cannot achieve its objective using either reward manipulation or action manipulation attack even in white-box setting, where the information about the MDP is assumed to be known. Hence, to contaminate a order-optimal RL algorithm, the adversary  needs to combine the power of reward manipulation and action manipulation. Indeed,  we show that  an attack that uses both reward manipulation and action manipulation can achieve adversary's objective with $\tilde{\Theta}(\sqrt{T})$ attack cost, which is order-optimal.
We also studied the in-feasibility of the attack under the constraint that the adversary can attack only if $a_t(h)\neq \pi^+_h(s)$. 
%Since this is a common constraint in the literature \cite{rakhsha2020policy} and some efficient attacks in MAB also satisfy this constraint \cite{jun2018adversarial,ma2019data,rangi2022saving}, it would be interesting to establish theoretically that this constraint is also necessary for designing an attack with sub-linear cost. 
%The study of infeasibility or feasibility of attack can be extended to dynamic manipulation attacks in RL. 
%Likewise, 
%An other interesting future direction is to study the feasibility of attack by manipulating the state observation. We believe that the ideas corresponding to action manipulation attack can be extended to study the corruption of the state perceived. 
%This is due to a key similarity between the two forms of attack, namely both these attacks affect the reward and the transition dynamics perceived by the learner simultaneously.%Finally, considering corruption of state opens up other attack options such as (state-reward) manipulation attack and (state-action)) manipulation, for a comprehensive study.
%On the other hand, our results reveals the vulnerability of order-optimal learning algorithms in RL. We hope this could spur more research on designing more robust algorithms for RL settings through ideas such as limited reward verification and corruption robustness \cite{lykouris2018stochastic,bogunovic2020stochastic}.  
%In addition, as our results reveal the vulnerability of order-optimal RL algorithms, designing more robust algorithms for RL settings would also be important.

\bibliographystyle{named}
\bibliography{ijcai22}

%\aistatsauthor{ Author 1 \And Author 2 \And  Author 3 }

%\aistatsaddress{ Institution 1 \And  Institution 2 \And Institution 3 } ]
\appendix 
\newpage
\onecolumn
\textbf{
\large{Supplementary Material: Understanding the Limits of   Poisoning Attacks in  Episodic Reinforcement Learning}}

\section{Proof of Theorem \ref{thm:infeasibleRewardMani}}\label{append:bounded-insufficiency}
\subsection{Insufficiency of (only) reward manipulation}
 
%\begin{example}\label{ex1:reward-insufficient}
In this example, $\mathcal{S}=\{s_1,s_2\}$ and $\mathcal{A}=\{a_1,a_2\}$. The transition dynamics is %\hf{maybe also say $H = 2$? }
\begin{equation}
    \mathbb{P}(s_1|s_1,a_1)=1, \mathbb{P}(s_2|s_1,a_1)=0, \mathbb{P}(s_1|s_1,a_2)=0, \mathbb{P}(s_2|s_1,a_2)=1,
\end{equation}
\begin{equation}
    \mathbb{P}(s_1|s_2,a_1)=0, \mathbb{P}(s_2|s_2,a_1)=1,  \mathbb{P}(s_1|s_2,a_2)=1, \mathbb{P}(s_2|s_2,a_2)=0.
\end{equation}
Also, we have
\begin{equation}\label{eq:meanExample}
    \mu(s_1,a_1)=\epsilon_1=0.25 , \mu(s_1,a_2)=1 , \mu(s_2,a_1)=\epsilon_2=.6, \mu(s_2,a_2)=1
\end{equation}
Let $H=2$. The target policy $\pi^+$ for the attacker is 
\begin{equation}
  \forall h\leq H:\quad  \pi^+_h(s_1)=a_1 \mbox{ and } \pi^+_h(s_2)=a_1.
\end{equation}
Similar to \cite{rakhsha2020policy}, the attacker is subject to following constraints
\begin{equation}\label{eq:AttackConstraint}
    r^o_t(s_t(h),a_t(h))= r_t(s_t(h),a_t(h)) \mbox{ if } a_t(h)= \pi^+_h(s_t(h)), \mbox{ and } r^o_t(s_t(h),a_t(h))\in [0,1],
\end{equation}
or equivalently, %\hf{I think we cannot assume this constraint, if we want to prove that all attack strategies will fail. Note that \cite{rakhsha2020policy} is proposing some attack strategies and for simplicity they restricted to the class of strategies that will not change rewards of $\pi^+$. However, this does not mean that for an attack strategy $\pi^+$ to succeed, it has to this. In particular, I think the following attack strategy --- which even manipulate the rewards of $\pi^+$ will be succeed ---- manipulate $r^o(s_1,a_2) = r^o(s_2,a_2) = r^o(s_2,a_1) = 0 $. All initial states are $s_1$. }\ar{I think if the attacker requires to modify the reward of $\pi^+$, then the cost cannot be sub-linear since the action in $\pi^+$ will be selected again and again. Hence, this constraint is critical to obtain a sub-linear cost, but I agree we don't have a proof for it. Also, the attack strategies proposed in MAB do not modify the target action, so the intuition is well supported by the literature. }
\begin{equation}
\begin{split}
     &\epsilon_t(s_t(h),a_t(h))= 0 \mbox{ if } a_t(h)= \pi^+_h(s_t(h)), \\
    \mbox{ and }& \epsilon_t(s_t(h),a_t(h))\in [-r_t(s_t(h),a_t(h)),1-r_t(s_t(h),a_t(h))].
\end{split}
\end{equation}
The objective of the attacker is that for all $\pi\neq \pi^+$, $h\leq H$ and $s\in \mathcal{S}$,
\begin{equation} 
    V_h^{\pi}(s)< V_h^{\pi^+}(s).
\end{equation}
Let policy $\tilde{\pi}$ be
\begin{equation}
\begin{split}
     h=1:&\qquad \tilde{\pi}_h(s_1)=a_2 \mbox{ and } \tilde{\pi}_h(s_2)=a_2,\\
     h=2:&\qquad \tilde{\pi}_h(s_1)=a_1 \mbox{ and } \tilde{\pi}_h(s_2)=a_1. 
\end{split}
\end{equation}
%\hf{do you mean $\tilde{\pi}_h(s_2)=a_2$? See my comment below}\ar{No. Please see the comment below}
At $h=H-1=1$, for all reward manipulation attack satisfying \eqref{eq:AttackConstraint}, we will show that
\begin{equation}
    V_h^{\tilde{\pi}}(s_1)> V_h^{{\pi}^+}(s_1).
\end{equation}
At $h=H-1=1$, we have
\begin{equation}\label{eq:policy1}
    V_h^{{\pi}^+}(s_1)=\mu(s_1,a_1)+ \mu(s_1,a_1)= 2\epsilon_1.
\end{equation}
Additionally, we have 
\begin{equation}\label{eq:policy2}
\begin{split}
    V_h^{\tilde{\pi}}(s_1)&\stackrel{(a)}{=}r^o_t(s_1,a_2)+\epsilon_2,\\
    &\stackrel{(b)}{\geq} \epsilon_2,
\end{split}
\end{equation}
where $(a)$ follows from the facts that $\tilde{\pi}_{H-1}(s_1)=a_2\neq \pi^+_H(s_1)$, which implies that the attacker can manipulate this observation, the next state at step $H$ is $s_2$, and  $\tilde{\pi}_{H}(s_2)=a_1= \pi^+_H(s_2)$, which implies that the attacker can not manipulate this observation, and $(b)$ follows from the fact that $r^o_t(s_1,a_2)\in [0,1]$ using \eqref{eq:AttackConstraint}. 

Now, since $\epsilon_1=0.25$ and $\epsilon_2=0.6$, comparing \eqref{eq:policy1} and \eqref{eq:policy2}, we have that 
\begin{equation}
    V_{H-1}^{\tilde \pi}(s_1)> V_{H-1}^{\pi^+}(s_1).
\end{equation}
This implies that for any reward manipulation attack in bounded setting satisfying \eqref{eq:AttackConstraint}, there exists a policy $\pi \neq \pi^+$, a step $h\leq H$ and a state $s_1\in \mathcal{S}$ such that
\begin{equation}
       V_{h}^{\pi}(s)> V_h^{\pi^+}(s).
\end{equation}
 
\subsection{Insufficiency of (only)  Action Manipulation}
 
% In this example, $\mathcal{S}=\{s_1,s_2\}$ and $\mathcal{A}=\{a_1,a_2\}$. The transition dymanics is
% \begin{equation}
%     \mathbb{P}(s_1|s_1,a_1)=1, \mathbb{P}(s_2|s_1,a_1)=0,
% \end{equation}
% \begin{equation}
%     \mathbb{P}(s_1|s_1,a_2)=0, \mathbb{P}(s_2|s_1,a_2)=1,
% \end{equation}
% \begin{equation}
%     \mathbb{P}(s_1|s_2,a_1)=0, \mathbb{P}(s_2|s_2,a_1)=1,
% \end{equation}
% \begin{equation}
%     \mathbb{P}(s_1|s_2,a_2)=1, \mathbb{P}(s_2|s_2,a_2)=0.
% \end{equation}
% Also, we have
% \begin{equation}\label{eq:meanExample}
%     \mu(s_1,a_1)=\epsilon_1=0.25 , \mu(s_1,a_2)=1 , \mu(s_2,a_1)=\epsilon_2=.6, \mu(s_2,a_2)=1. 
% \end{equation}
% The target policy $\pi^+$ for the attacker is 
% \begin{equation}
%   \forall h\leq H:\quad  \pi^+_h(s_1)=a_1 \mbox{ and } \pi^+_h(s_2)=a_2.
% \end{equation}
% The attacker is subject to following constraints
% \begin{equation}\label{eq:AttackConstraint_dm}
%     a^o_t(h)= a_t(h) \mbox{ if } a_t(h)= \pi^+_h(s_t(h)),
% \end{equation}
 
% The objective of the attacker is that for all $\pi\neq \pi^+$, $h\leq H$ and $s\in \mathcal{S}$,
% \begin{equation} 
%     V_h^{\pi}(s)< V_h^{\pi^+}(s).
% \end{equation}

The MDP construction and target policy $\pi^+$ are the same as the one in Theorem \ref{thm:infeasibleRewardMani}. Let $H=2$. We also consider policy $\tilde{\pi}$ as follows
\begin{equation}
\begin{split}
     h=1:&\qquad \tilde{\pi}_h(s_1)=a_2 \mbox{ and } \tilde{\pi}_h(s_2)=a_2,\\
     h=2:&\qquad \tilde{\pi}_h(s_1)=a_1 \mbox{ and } \tilde{\pi}_h(s_2)=a_1. 
\end{split}
\end{equation}
At $h=H-1=1$, for all reward manipulation attack satisfying \eqref{eq:AttackConstraint}, we will show that
\begin{equation}
    V_h^{\tilde{\pi}}(s_1)> V_h^{{\pi}^+}(s_1).
\end{equation}
At $h=H-1=1$, we have
\begin{equation}\label{eq:policy1_1}
    V_h^{{\pi}^+}(s_1)=\mu(s_1,a_1)+ \mu(s_1,a_1)= 2\epsilon_1.
\end{equation}
Additionally, we have 
\begin{equation}\label{eq:policy2_1}
\begin{split}
    V_h^{\tilde{\pi}}(s_1)&\stackrel{(a)}{\geq}\min\{\mu(s_1,a_1)+\mu(s_1,a_1), \mu(s_1,a_2)+\mu(s_2,a_1)\}\\
    &\stackrel{(b)}{\geq} \min\{2\epsilon_1,1+\epsilon_2\}
\end{split}
\end{equation}
where $(a)$ follows from the facts that $\tilde{\pi}_{H-1}(s_1)=a_2\neq \pi^+_H(s_1)$, which implies that the attacker can manipulate the action, namely $a_t(h)=a_1$ or $a_t(h)=a_2$, and if $a_t(h)=a_1$ (or $a_t(h)=a_2$), then $V_h^{\tilde{\pi}}(s_1)$ is $2\mu(s_1,a_1)$ (or $\mu(s_1,a_2)+\mu(s_2,a_1)$), 
 and $(b)$ follows from \eqref{eq:meanExample}. 

Now, since $\epsilon_1=0.25$ and $\epsilon_2=0.6$, comparing \eqref{eq:policy1_1} and \eqref{eq:policy2_1}, we have that 
\begin{equation}
    V_{H-1}^{\tilde \pi}(s_1)\geq V_{H-1}^{\pi^+}(s_1).
\end{equation}
This implies that for any action manipulation attack in bounded setting satisfying \eqref{eq:AttackConstraint}, there exists a policy $\pi \neq \pi^+$, a step $h\leq H$ and a state $s_1\in \mathcal{S}$ such that
\begin{equation}
       V_{h}^{\pi}(s)\geq V_h^{\pi^+}(s).
\end{equation}

\section{White-box Attack in Bounded Reward Setting}\label{append:white-bounded}
The attack strategy is 
\begin{equation}\label{eq:trasitionCorruption_in}
    a^o_t(h)=\begin{cases} a_t(h)&\mbox{ if } a_t(h)= \pi_h^+(s),\\
      \pi_h^+(s) &\mbox{ if }  a_t(h)\neq \pi_h^+(s),
    \end{cases}
\end{equation}
and
{\scriptsize
\begin{equation}\label{eq:RL_fullInfo2_in}
    r_t^o(s_t(h), a_t(h))=\begin{cases} r_t(s_t(h),a_t(h))&\mbox{ if } a_t(h)= \pi_h^+(s_t(h)),\\
      \tilde{Q}_h^{\pi^+}(s_t(h),\pi_h^+(s_t(h)))-E_{s^\prime\sim \mathbb{P}(s^\prime|s_t(h),\pi_h^+(s_t(h)))}[\bar{V}^{\pi^+}_{h+1}(s^\prime)]-\epsilon& otherwise.
    \end{cases}
\end{equation}}
where $\bar{Q}^\pi_h(s,a)$ is the expected reward in state $s$ for action $a$ introduced by the above reward and action manipulation under policy $\pi$, and $\bar{V}_h^\pi(s)$ is the expected reward in state $s$ for the above reward and action manipulation under policy $\pi$.

\begin{theorem}\label{thm:boundedWhiteBox}
For any learning algorithm $\mathcal{A}$ such that for all $T\geq t_0$, the regret in the absence of attack is 
\begin{equation}\label{eq:regA_12}
    R^{\mathcal{A}}(T,H)=\tilde O(\sqrt{T}H^{\alpha}),
\end{equation}
with probability at least $1-\delta$, where $\alpha\geq 1$ is a numerical constant; and for any sub-optimal target policy $\pi^+$ and $0<\epsilon\leq \min_{h\leq H, s\in\mathcal{S}}\mu(s,\pi^+_h(s))$, if an attacker follows strategy in \eqref{eq:trasitionCorruption_in} and \eqref{eq:RL_fullInfo2_in}, then with probability at least $1-\delta$, the number of reward manipulation attacks will be
 \begin{equation}\label{eq:NumFinalRes_1.2}
     \sum_{t=1}^T\sum_{h=1}^H\mathbf{1}(\epsilon_{t,h}(s_t(h),a_t(h))\neq 0)=\tilde O\big({\sqrt{T}H^\alpha}/{\epsilon}\big),
 \end{equation}
 the amount of contamination is 
\begin{equation}\label{eq:conntaminBB_1.2}
    \sum_{t=1}^T\sum_{h=1}^H|\epsilon_{t,h}(s_t(h),a_t(h)|= \tilde O\big(\sqrt{T}H^{\alpha}/\epsilon\big),
\end{equation}
the number of action manipulation attacks is 
\begin{equation}\label{eq:numActionMannipulation}
    \sum_{t=1}^{T}\sum_{h=1}^H\mathbf{1}(a^o_t(h)\neq a_t(h))=\tilde O\big({\sqrt{T}H^\alpha}/{\epsilon}\big),
\end{equation}
and $\sum_{t=1}^T\sum_{h=1}^H\mathbf{1}(a_t(h)= \pi_h^+(s_{t}(h)))=\Omega(T)$. 
\end{theorem}
\begin{proof}
First, we will show that the optimal policy under action and reward manipulation attack in \eqref{eq:trasitionCorruption_in} and \eqref{eq:RL_fullInfo2_in} is $\pi^+$, namely  for all $\pi\neq\pi^+$, $h\leq H$ and $s\in \mathcal{S}$, we have
\begin{equation}\label{eq:valueFunct_1_1}
    \bar{V}_h^{\pi^+}(s)> \bar{V}_h^\pi (s).
\end{equation}
We will show this by induction. We will that that \eqref{eq:valueFunct_1_1} holds for $h=H$. Then, we will show that \eqref{eq:valueFunct_1_1} holds for $h<H$  if it holds for $h+1$. 
At $h=H$, for all $\pi$, using \eqref{eq:trasitionCorruption_in} and \eqref{eq:RL_fullInfo2_in}, we have that  
\begin{equation}
    \bar{Q}_{H}^\pi(s,a)=\begin{cases} \mu(s,a)&\mbox{ if } a=\pi_{H}^+(s),\\
      \mu (s,\pi_{H}^+(s))-\epsilon&\mbox{ if } otherwise,
    \end{cases}
\end{equation}
since episode terminates at step $H$. This implies that for $h=H$, we have that \eqref{eq:valueFunct_1_1} holds, and for all $(s,a)\in \mathcal{S}\times \mathcal{A}$ such that $a\neq \pi_H^+(s)$, we have
\begin{equation}
    \bar{Q}_{H}^\pi(s,a)= \bar{Q}_{H}^{\pi^+}(s,\pi_H^+(s))-\epsilon.
\end{equation}
Now, consider any $h<H$. Let \eqref{eq:valueFunct_1_1} holds for $h+1$. Using \eqref{eq:trasitionCorruption_in} and \eqref{eq:RL_fullInfo2_in}, for all $\pi$, we have that
{
\begin{equation}\label{eq:QValue_1}
    \bar{Q}^\pi_{h}(s,a)=\begin{cases} \mu (s,a)+\mathbb{E}_{s^\prime\sim P(s^\prime|s,a)}[\bar{V}^\pi_{h+1}(s^\prime)]&\mbox{ if } a= \pi_{h}^+(s),\\
        \bar{Q}^{\pi^+}_{h}(s, \pi_{h}^+(s))-\mathbb{E}_{s^\prime\sim P(s^\prime|s,\pi_{h}^+(s))}[\bar V_{h+1}^{\pi^+}(s^\prime)]+ \mathbb{E}_{s^\prime\sim P(s^\prime|s,\pi_{h}^+(s))}[\bar{V}^\pi_{h+1}(s^\prime)]-\epsilon&\mbox{ if } otherwise.
    \end{cases}
\end{equation}
}
Since \eqref{eq:valueFunct_1_1} holds for $h+1$, we have that for $a=\pi_h^+(s)$,
\begin{equation}
    \bar{Q}^\pi_{h}(s,a)< \mu (s,a)+\mathbb{E}_{s^\prime\sim P(s^\prime|s,a)}[\bar{V}^{\pi^+}_{h+1}(s^\prime)]=\bar{Q}^{\pi^+}_{h}(s,a).
\end{equation}
Additionally, for $a\neq\pi^+_{h}(s)$, we have
\begin{equation}\label{eq:qfunction1_1}
\begin{split}
      \bar{Q}^\pi_{h}(s,a)&= \bar{Q}^{\pi^+}_{h}(s, \pi_{h}^+(s))-\mathbb{E}_{s^\prime\sim P(s^\prime|s,\pi_{h}^+(s))}[\bar V_{h+1}^{\pi^+}(s^\prime)]+ \mathbb{E}_{s^\prime\sim P(s^\prime|s,\pi_{h}^+(s))}[\bar{V}^\pi_{h+1}(s^\prime)]-\epsilon,\\
      &=\bar{Q}^{\pi^+}_{h}(s, \pi_{h}^+(s))+ \mathbb{E}_{s^\prime\sim P(s^\prime|s,\pi_{h}^+(s))}[\bar{V}^\pi_{h+1}(s^\prime)- \bar{V}_{h+1}^{\pi^+}(s^\prime)]-\epsilon,\\
      &\stackrel{(a)}{<} \bar{Q}^{\pi^+}_{h}(s, \pi_{h}^+(s))-\epsilon,
\end{split}
\end{equation}
where $(a)$ follows from the fact that \eqref{eq:valueFunct_1_1} holds for $h+1$. Hence, the first step of the proof follows. 

Additionally, the attack satisfies the constraint that $r_t^o(s, a_t(h))\in [0,1]$. For $a_t(h)\neq \pi^+_h(s_t(h))$, we have
\begin{equation}
\begin{split}
    r_t^o(s_t(h), a_t(h))&=\bar{Q}^{\pi^+}_{h}(s_t(h), \pi_{h}^+(s_t(h)))-\mathbb{E}_{s^\prime\sim P(s^\prime|s_t(h),\pi_{h}^+(s_t(h)))}[\bar V_{h+1}^{\pi^+}(s^\prime)]-\epsilon,\\
    &\stackrel{(a)}{=}\mu(s_t(h),\pi_{h}^+(s_t(h)))-\epsilon,\\
    &\stackrel{(b)}{\geq}0,
\end{split}
\end{equation}
where $(a)$ follows from the fact that
\begin{equation}
    \bar{Q}^{\pi^+}_{h}(s_t(h), \pi_{h}^+(s_t(h)))=\mu(s_t(h),\pi_{h}^+(s_t(h)))+E_{s^\prime\sim P(s^\prime|a_t(h),\pi_{h}^+(s_t(h)))}[\bar V_{h+1}^{\pi^+}(s^\prime)],
\end{equation}
and $(b)$ follows from the fact that $0<\epsilon\leq \min_{h\leq H, s\in\mathcal{S}}\mu(s,\pi^+_h(s))$. Additionally, we have 
\begin{equation}
    r_t^o(s_t(h),a_t(h))=\mu(s_t(h),a_t(h))-\epsilon\leq 1,
\end{equation}
since $\mu(s_t(h),a_t(h))\in (0,1]$. 

Let $\Delta(a)=\min_{s,h,\pi} \bar{Q}_h^{\pi^+}(s,\pi^+_h(s))-\bar{Q}_h^{\pi}(s,a)$. Using \eqref{eq:qfunction1_1}, we have that 
\begin{equation}\label{eq:deltaLB_1_2}
\Delta(a)\geq \epsilon.     
\end{equation}
Now, using \eqref{eq:deltaLB_1_2}, we have that
\begin{equation}\label{eq:subOpttimal_1_1}
\begin{split}
        \sum_{t=1}^{T}\sum_{h=1}^H\epsilon\mathbf{1}(a_t(h)\neq \pi_h^+(s_t(h)))&\leq R^{\mathcal{A}}(T,H),\\
        &=\tilde O(\sqrt{T}H^\alpha),
\end{split}
\end{equation}
with probability $1-\delta$, where the last inequality follows from \eqref{eq:regA_12}. This along with \eqref{eq:trasitionCorruption_in} and \eqref{eq:RL_fullInfo2_in} implies that \eqref{eq:NumFinalRes_1.2} and \eqref{eq:numActionMannipulation} follows since the contamination happens only if $a_t(h)\neq \pi_h^+(s)$.

Additionally, for all $h\leq H$ and $(s_t(h),a_t(h))\in \mathcal{S}\times \mathcal{A}$ such that $a_t(h)\neq \pi^+_h(s)$,  we have that the amount of contamination is at most one, which implies \eqref{eq:conntaminBB_1.2} follows.

Finally, we have that with probability $1-\delta$
\begin{equation}
\begin{split}
       \sum_{t=1}^T\sum_{h=1}^H\mathbf{1}(a_t(h)= \pi_h^+(s_{t}(h)))&=TH-\sum_{t=1}^T\sum_{h=1}^H\mathbf{1}(a_t(h)\neq \pi_h^+(s_{t}(h)))\\
       &=\Omega(T),
\end{split}
\end{equation}
where the last equality follows from \eqref{eq:subOpttimal_1_1}.  Hence, the statement of the theorem follows. 
\end{proof}

%\section{Miscellaneous results}

\section{Proof of Theorem \ref{thm:boundedBlackBox}}
\begin{proof}
First, we will show that the optimal policy under action and reward manipulation attack in \eqref{eq:trasitionCorruption_in} and \eqref{eq:RL_fullInfo2_in} is $\pi^+$, namely  for all $\pi\neq\pi^+$, $h\leq H$ and $s\in \mathcal{S}$, we have
\begin{equation}\label{eq:valueFunct_1_2}
    \bar{V}_h^{\pi^+}(s)> \bar{V}_h^\pi (s).
\end{equation}
We will show this by induction. We will that that \eqref{eq:valueFunct_1_2} holds for $h=H$. Then, we will show that \eqref{eq:valueFunct_1_2} holds for $h<H$  if it holds for $h+1$. 
At $h=H$, for all $\pi$, using  \eqref{eq:trasitionCorruption_unin} and \eqref{eq:RL_fullInfo2_uninb}, we have that  
\begin{equation}
    \bar{Q}_{H}^\pi(s,a)=\begin{cases} \mu(s,a)&\mbox{ if } a=\pi_{H}^+(s),\\
      0 & otherwise,
    \end{cases}
\end{equation}
since episode terminates at step $H$. This implies that for $h=H$, we have that \eqref{eq:valueFunct_1_2} holds, and for all $(s,a)\in \mathcal{S}\times \mathcal{A}$ such that $a\neq \pi_H^+(s)$, we have
\begin{equation}
    \bar{Q}_{H}^\pi(s,a)= 0.
\end{equation}
Now, consider any $h<H$. Let \eqref{eq:valueFunct_1_2} holds for $h+1$. Using  \eqref{eq:trasitionCorruption_unin} and \eqref{eq:RL_fullInfo2_uninb}, for all $\pi$, we have that
{
\begin{equation}\label{eq:QValue_2}
    \bar{Q}^\pi_{h}(s,a)=\begin{cases} \mu (s,a)+\mathbb{E}_{s^\prime\sim P(s^\prime|s,a)}[\bar{V}^\pi_{h+1}(s^\prime)]&\mbox{ if } a= \pi_{h}^+(s),\\
        0+ \mathbb{E}_{s^\prime\sim P(s^\prime|s,\pi_{h}^+(s))}[\bar{V}^\pi_{h+1}(s^\prime)]&\mbox{ if } otherwise.
    \end{cases}
\end{equation}
}
Since \eqref{eq:valueFunct_1_2} holds for $h+1$, we have that for $a=\pi_h^+(s)$,
\begin{equation}
    \bar{Q}^\pi_{h}(s,a)<\mu (s,a)+\mathbb{E}_{s^\prime\sim P(s^\prime|s,a)}[\bar{V}^{\pi^+}_{h+1}(s^\prime)]=\bar{Q}^{\pi^+}_{h}(s,a).
\end{equation}
Additionally, for $a\neq\pi^+_{h}(s)$, we have
\begin{equation}\label{eq:qfunction1_2}
\begin{split}
      \bar{Q}^\pi_{h}(s,a)&= \mathbb{E}_{s^\prime\sim P(s^\prime|s,\pi_{h}^+(s))}[\bar{V}^\pi_{h+1}(s^\prime)],\\
      &\stackrel{(a)}{<} \mathbb{E}_{s^\prime\sim P(s^\prime|s,\pi_{h}^+(s))}[\bar{V}^{\pi^+}_{h+1}(s^\prime)],\\
      &\stackrel{(b)}{<} \bar{Q}^{\pi^+}_{h}(s, \pi_{h}^+(s)),
\end{split}
\end{equation}
where $(a)$ follows from the fact that \eqref{eq:valueFunct_1_2} holds for $h+1$, and $(b)$ follows from the definition of $\bar{Q}^{\pi^+}_{h}(s, \pi_{h}^+(s))$. Hence, the first step of the proof follows. 

Additionally, the attack satisfies the constraint that $r_t^o(s, a_t(h))\in [0,1]$. %For $a_t(h)\neq \pi^+_h(s_t(h))$, we have
%\begin{equation}
%\begin{split}
%    r_t^o(s_t(h), a_t(h))&=\bar{Q}^{\pi^+}_{h}(s_t(h), \pi_{h}^+(s_t(h)))-\mathbb{E}_{s^\prime\sim P(s^\prime|s_t(h),\pi_{h}^+(s_t(h)))}[\bar V_{h+1}^{\pi^+}(s^\prime)]-\epsilon,\\
%    &\stackrel{(a)}{=}\mu(s_t(h),\pi_{h}^+(s_t(h)))-\epsilon,\\
%    &\stackrel{(b)}{\geq}0,
%\end{split}
%\end{equation}
%where $(a)$ follows from the fact that
%\begin{equation}
%    \bar{Q}^{\pi^+}_{h}(s_t(h), \pi_{h}^+(s_t(h)))=\mu(s_t(h),\pi_{h}^+(s_t(h)))+E_{s^\prime\sim P(s^\prime|a_t(h),\pi_{h}^+(s_t(h)))}[\bar V_{h+1}^{\pi^+}(s^\prime)],
%\end{equation}
%and $(b)$ follows from the fact that $0<\epsilon\leq \min_{h\leq H, s\in\mathcal{S}}\mu(s,\pi^+_h(s))$. %Additionally, we have 
%\begin{equation}
%    r_t^o(s_t(h),a_t(h))=\mu(s_t(h),a_t(h))-\epsilon\leq 1,
%\end{equation}
%since $\mu(s_t(h),a_t(h))\in (0,1]$. 

Let $\Delta(a)=\min_{s,h,\pi} \bar{Q}_h^{\pi^+}(s,\pi^+_h(s))-\bar{Q}_h^{\pi}(s,a)$. Using the fact that $r^o_t(s_t(h),a_t(h))=0$ if $a_t(h)\neq \pi^+_h(s_t(h))$ , we have that 
\begin{equation}\label{eq:deltaLB_1_1}
\Delta(a)\geq \min_{h,s}\mu(s,\pi^+_h(s)).     
\end{equation}
Now, using \eqref{eq:deltaLB_1_1}, we have that
\begin{equation}\label{eq:subOpttimal_1_2}
\begin{split}
        \sum_{t=1}^{T}\sum_{h=1}^H\min_{h,s}\mu(s,\pi^+_h(s))\mathbf{1}(a_t(h)\neq \pi_h^+(s_t(h)))&\leq R^{\mathcal{A}}(T,H),\\
        &=\tilde O(\sqrt{T}H^\alpha),
\end{split}
\end{equation}
with probability $1-\delta$, where the last inequality follows from \eqref{eq:regA_13}. This along with \eqref{eq:trasitionCorruption_unin} and \eqref{eq:RL_fullInfo2_uninb} implies that  part $2.$ and $4.$ follows since the contamination happens only if $a_t(h)\neq \pi_h^+(s)$.

Additionally, for all $h\leq H$ and $(s_t(h),a_t(h))\in \mathcal{S}\times \mathcal{A}$ such that $a_t(h)\neq \pi^+_h(s)$,  we have that the amount of contamination is at most one, which implies part $3.$ follows.

Finally, we have that with probability $1-\delta$,
\begin{equation}
\begin{split}
       \sum_{t=1}^T\sum_{h=1}^H\mathbf{1}(a_t(h)= \pi_h^+(s_{t}(h)))&=TH-\sum_{t=1}^T\sum_{h=1}^H\mathbf{1}(a_t(h)\neq \pi_h^+(s_{t}(h)))\\
       &=\Omega(T),
\end{split}
\end{equation}
where the last equality follows from \eqref{eq:subOpttimal_1_2}.  Hence, the statement of the theorem follows. 
\end{proof}

%For the camera-ready paper, if you are using \LaTeX, please make sure
%that you follow these instructions.  
% (If you are not using \LaTeX,
%please make sure to achieve the same effect using your chosen
%typesetting package.)

\section{White-Box Attack in Unbounded Reward Setting}\label{appendix:unbounded-white}
In white-box attack setting, the attacker posses the knowledge about the expected reward and the transition dynamics of the MDP. In this section, we propose a whitebox attack which utilizes this information about the MDP, and achieves an order optimal attack cost. This attack is different from the white-box attack proposed in  \cite{rakhsha2020policy}, and is adapted to the white-box setting in episodic RL. 

Given the target policy   $\pi^+$ and an input parameter $\epsilon>0$, for all $s_t(h)\in \mathcal{S}$, $a_t(h)\in\mathcal{A}$ and $h\leq H$, our proposed attack strategy is  
{
\begin{equation}\label{eq:attackStrategy_RL_fullInfo2}
    r_t^o(s_t(h), a_t(h))=\begin{cases} r_t(s_t(h),a_t(h))&\mbox{ if } a_t(h)= \pi_{h}^+(s_t(h)),\\
      \tilde{Q}^{\pi^+}_h(s_t(h), \pi_{h}^+(s_t(h)))-\mathbb{E}_{s^\prime\sim P(s^\prime|s_{t}(h),a_t(h))}[\tilde V_{h+1}^{\pi^+}(s^\prime)]-\epsilon& otherwise.
    \end{cases}
\end{equation}
}
%\hf{add a discussion about how to find such $r_t^o$ via backward induction and linear programs.}
%Does it always exist? Your definition seems lead to a fixed point problem because $\tilde{Q}^{\pi^+_h}$ depends on $r_t^o$ as well. It is not clear this fixed point problem will have a solution.....}\ar{I think $r_t^o$ can be computed for a deterministic policy $\pi^+$ using backward recursion starting from $h=H$. The dynamics are already known.  }
where $\tilde{Q}^{\pi}_h(s,a)$ is the expected reward in state $s$ for action $a$ for the above reward observation under policy $\pi$, and $\tilde{V}^{\pi}_h(s)$ is the expected reward in state $s$ for the above reward observation under policy $\pi$. These values will not be same as the ones defined in \eqref{eq:ValueFunc} and \eqref{eq:Qfunnc} since the reward observations are manipulated. We remark that the $r_t^o(s_t(h), a_t(h))$ can be computed through a backward induction procedure starting from horizon $H$. At any step $h$ in the episode, the definition of $r_t^o(s_t(h), a_t(h))$ depends linearly on the Q-values at $h$, which then depends linearly on $r_t^o(s_t(h), a_t(h))$. Therefore, $r_t^o(s_t(h), a_t(h))$ at any horizon $h$ can be computed by solving a linear system. 

During the attack in \eqref{eq:attackStrategy_RL_fullInfo2}, the reward observations are manipulated only if the action selected by the learner is not the same as the desired action by $\pi^+$. The above reward manipulation strategy ensures that the target policy $\pi^+$ is the optimal policy based on the observed reward observations, namely for all $h\leq H$, and $(s,a)\in \mathcal{S}\times\mathcal{A}$ such that $a\neq \pi_h^+(s)$, we have
\begin{equation}
    \tilde{Q}_{h}(s,a)\leq \tilde{Q}_{h}(s,\pi^+_h(s))-\epsilon. 
\end{equation}
This implies that the parameter $\epsilon$ in the above attack can be tuned to obtain a desired difference between the expected rewards of the optimal policy and any other policy. This requirement been studied in form of $\epsilon$-robust policy in \cite{rakhsha2020policy}. 

Following theorem provides an upper bound on the amount of contamination for an order-optimal learning algorithm.
\begin{theorem}\label{thm:unBoundedWhiteBox}%[Restatement of Theorem \ref{thm:unBoundedWhiteBox}]
For any learning algorithm whose regret in the absence of attack is given by
\begin{equation}\label{eq:regretUB}
    R^{\mathcal{A}}(T,H)=\tilde O(\sqrt{T} H^\alpha), 
\end{equation}
with probability at least $1-\delta$, where $\alpha\geq 1$ is a numerical constant; and for any sub-optimal target policy $\pi^+$ and $\epsilon>0$, if an attacker follows strategy \eqref{eq:attackStrategy_RL_fullInfo2}, then with probability at least $1-\delta$, the number of reward manipulation attacks will be 
\begin{equation}\label{eq:result1}
    \sum_{t=1}^T\sum_{h=1}^H \mathbf{1}(\epsilon_{t,h}(s_t(h),a_t(h))\neq 0)=\tilde O\big({\sqrt{T} H^\alpha}/{\epsilon}\big),
\end{equation}
the amount of contamination
\begin{equation}\label{eq:amounntContamination}
  \sum_{t=1}^T\sum_{h=1}^H |\epsilon_{t,h}(s_t(h),a_t(h))|= \tilde{O}\big(\sqrt{T} H^{\alpha+1}+{\sqrt{T} H^{\alpha+1}}/{\epsilon}\big),
\end{equation}
$\sum_{t=1}^T\sum_{h=1}^H\mathbf{1}(a_t(h)= \pi_h^+(s_{t}(h)))=\Omega(T)$, and attacker achieves its objective in \eqref{eq:ObjectAttacker}. 
\end{theorem}
\begin{proof}
First, we will show that the optimal policy under the reward manipulation attack in \eqref{eq:attackStrategy_RL_fullInfo2} is $\pi^+$, namely for all $\pi\neq\pi^+$, $h\leq H$ and $s\in \mathcal{S}$, we have
\begin{equation}\label{eq:valueFunct}
    \tilde{V}_h^{\pi^+}(s)> \tilde{V}_h^\pi (s).
\end{equation}
We will show this by induction. We will show that \eqref{eq:valueFunct} holds for $h=H$. Then, we will show that \eqref{eq:valueFunct} holds for $h<H$  if it holds for $h+1$. 
At $h=H$, for all $\pi$, using \eqref{eq:attackStrategy_RL_fullInfo2},  we have that 
\begin{equation}
    \tilde{Q}_{H}^\pi(s,a)=\begin{cases} \mu(s,a)&\mbox{ if } a=\pi_{H}^+(s),\\
      \mu (s,\pi_{H}^+(s))-\epsilon&\mbox{ if } otherwise.
    \end{cases}
\end{equation}
This implies that for $h=H$, we have that \eqref{eq:valueFunct} holds, and for all $(s,a)\in \mathcal{S}\times \mathcal{A}$ such that $a\neq \pi_H^+(s)$, we have
\begin{equation}
    \tilde{Q}_{H}^\pi(s,a)= \tilde{Q}_{H}^{\pi^+}(s,\pi_H^+(s))-\epsilon.
\end{equation}
Now, consider any $h<H$. Let \eqref{eq:valueFunct} holds for $h+1$. 
Using \eqref{eq:attackStrategy_RL_fullInfo2}, for all $\pi$, we have that 
{
\begin{equation}\label{eq:QValue}
    \tilde{Q}^\pi_{h}(s,a)=\begin{cases} \mu (s,a)+\mathbb{E}_{s^\prime\sim P(s^\prime|s,a)}[\tilde{V}^\pi_{h+1}(s^\prime)]&\mbox{ if } a= \pi_{h}^+(s),\\
        \tilde{Q}^{\pi^+}_{h}(s, \pi_{h}^+(s))-\mathbb{E}_{s^\prime\sim P(s^\prime|s,a)}[\tilde V_{h+1}^{\pi^+}(s^\prime)]+ \mathbb{E}_{s^\prime\sim P(s^\prime|s,a)}[\tilde{V}^\pi_{h+1}(s^\prime)]-\epsilon&\mbox{ if } otherwise.
    \end{cases}
\end{equation}
}
Since \eqref{eq:valueFunct} holds for $h+1$, we have that for $a=\pi^+_{h}(s)$
\begin{equation}
    \tilde{Q}^\pi_{h}(s,a)< \mu (s,a)+\mathbb{E}_{s^\prime\sim P(s^\prime|s,a)}[\tilde{V}^{\pi^+}_{h+1}(s^\prime)]=\tilde{Q}^{\pi^+}_{h}(s,a).
\end{equation}
Additionally, for $a\neq\pi^+_{h}(s)$, we have
\begin{equation}\label{eq:qfunction1}
\begin{split}
      \tilde{Q}^\pi_{h}(s,a)&= \tilde{Q}^{\pi^+}_{h}(s, \pi_{h}^+(s))-\mathbb{E}_{s^\prime\sim P(s^\prime|s,a)}[\tilde V_{h+1}^{\pi^+}(s^\prime)]+ \mathbb{E}_{s^\prime\sim P(s^\prime|s,a)}[\tilde{V}^\pi_{h+1}(s^\prime)]-\epsilon,\\
      &=\tilde{Q}^{\pi^+}_{h}(s, \pi_{h}^+(s))+ \mathbb{E}_{s^\prime\sim P(s^\prime|s,a)}[\tilde{V}^\pi_{h+1}(s^\prime)- \tilde{V}_{h+1}^{\pi^+}(s^\prime)]-\epsilon,\\
      &\stackrel{(a)}{<} \tilde{Q}^{\pi^+}_{h}(s, \pi_{h}^+(s))-\epsilon,
\end{split}
\end{equation}
where $(a)$ follows from the fact that \eqref{eq:valueFunct} holds for $h+1$. Hence, the first step of the proof follows. 

Let $\Delta(a)=\min_{s,h,\pi} \tilde{Q}_h^{\pi^+}(s,\pi^+_h(s))-\tilde{Q}_h^{\pi}(s,a)$. Using \eqref{eq:qfunction1}, for $a\neq \pi^+_h(s)$,   we have that 
\begin{equation}\label{eq:deltaLB}
\Delta(a)\geq \epsilon.     
\end{equation}
Now, using \eqref{eq:deltaLB}, we have that
\begin{equation}\label{eq:subOpttimal}
\begin{split}
        \sum_{t=1}^{T}\sum_{h=1}^H\epsilon\mathbf{1}(a_t(h)\neq \pi_h^+(s_t(h)))&\leq R^{\mathcal{A}}(T,H),\\
        &=\tilde O(\sqrt{T}H^\alpha),
\end{split}
\end{equation}
with probability $1-\delta$, where the last inequality follows from \eqref{eq:regretUB}. This along with \eqref{eq:attackStrategy_RL_fullInfo2} implies that \eqref{eq:result1} follows since the contamination happens only if $a_t(h)\neq \pi_h^+(s)$. 

Additionally, for all $h\leq H$ and $(s_t(h),a_t(h))\in \mathcal{S}\times \mathcal{A}$ such that $a_t(h)\neq \pi^+_h(s)$,  we have that the amount of contamination is
\begin{equation}
\begin{split}
   &\big| \tilde{Q}^{\pi^+}_h(s_t(h), \pi_{h}^+(s_t(h)))-\mathbb{E}_{s^\prime\sim P(s^\prime|s_{t}(h),a_t(h))}[\tilde V_{h+1}^{\pi^+}(s^\prime)]-\epsilon -\mu(s_t(h),a_t(h))]\big|\\
   &\leq \big| \tilde{Q}^{\pi^+}_h(s_t(h), \pi_{h}^+(s_t(h)))-\mathbb{E}_{s^\prime\sim P(s^\prime|s_{t}(h),a_t(h))}[\tilde V_{h+1}^{\pi^+}(s^\prime)]\big|+\epsilon +\max_{s,a}\big|\mu(s,a)\big|\\
   &\leq (H+1)\max_{s,a}\big|\mu(s,a)\big|+\epsilon.
\end{split}
\end{equation}
This along with \eqref{eq:subOpttimal} implies that with probability $1-\delta$,
\begin{equation}
    \sum_{t=1}^{T}\sum_{h=1}^H|\epsilon_{t,h}(s_t(h),a_t(h))|= \tilde O\bigg(((H+1)\max_{s,a}\big|\mu(s,a)\big|+\epsilon)\frac{\sqrt{T}H^\alpha}{\epsilon}\bigg).
\end{equation}
Hence, we have that \eqref{eq:amounntContamination} follows. Finally, we have that with probability $1-\delta$
\begin{equation}
\begin{split}
       \sum_{t=1}^T\sum_{h=1}^H\mathbf{1}(a_t(h)= \pi_h^+(s_{t}(h)))&=TH-\sum_{t=1}^T\sum_{h=1}^H\mathbf{1}(a_t(h)\neq \pi_h^+(s_{t}(h)))\\
       &=\Omega(T),
\end{split}
\end{equation}
where the last equality follows from \eqref{eq:subOpttimal}. Hence, the statement of the theorem follows. 

\end{proof}

Similar to the white-box attack in \cite{rakhsha2020policy}, Theorem \ref{thm:unBoundedWhiteBox} shows that the attacker can achieve its objective in  $\tilde O(\sqrt{T})$ attack cost in episodic RL. The attack cost is of the same order as the lower bound in \cite[Theorem 1]{rakhsha2020policy}.

\section{Proof of Theorem \ref{thm:BBAtackun}}\label{append:thm-proof-unbounded-black}
The following lemma will be used in our proof. 

\begin{lemma}\label{lemma:ucbLcbRL}
For all $(s,a)\in \mathcal{S}\times\mathcal{A}$, $T\geq 1$ and $H\geq 1$, we have that
\begin{equation}
    \mathbb{P}\bigg(|\hat\mu(s,a)-\mu(s,a)|> \sigma\sqrt{4\log(2THSA)/N_{t,h}(s,a)}\bigg)\leq  \frac{1}{(THSA)^2}.
\end{equation}
%or equivalently with probability at least $1-{1}/{(THSA)^2}$,
%\begin{equation}
%  \hat{\mu}^{LCB}_{t,h}(s,a)\leq  \mu(s,a)\leq \hat{\mu}^{UCB}_{t,h}(s,a).
%\end{equation}
\end{lemma}
\begin{proof}
The above lemma following using concentration inequality for sub-gaussian random variable. 
\end{proof}

\begin{proof}[Proof of Theorem \ref{thm:BBAtackun}]
Since $\mu(s,a) \in [-M, M]$ is bounded, if $N(s,a)=0$, then we have
\begin{equation}\label{eq:con111}
    \hat{\mu}^{LCB}(s,a)\leq \mu(s,a)\leq \hat{\mu}^{UCB}(s,a).
\end{equation}

Using the fact that reward observation are $\sigma^2$-subgaussian random variable, we have that
\begin{equation}\label{eq:con111-2}
    \mathbb{P}\bigg(|r_t^o(s,a)-\mu(s,a)|>\sigma\sqrt{4\log(2HSAT)}\bigg)\leq 1/(HSAT)^2. 
\end{equation}
 Let the event
\begin{equation}
    \mathcal{E}_2=\{\forall 1\leq t\leq T, h\leq H, s\in \mathcal{S},a\in\mathcal{A}: \hat{\mu}^{LCB}_{t,h}(s,a)\leq \mu(s,a)\leq \hat{\mu}_{t,h}^{UCB}(s,a)\},
\end{equation}
where $\hat{\mu}^{LCB}_{t,h}(s,a)$ and $\hat{\mu}_{t,h}^{UCB}(s,a)$ is the estimate $\hat{\mu}^{LCB}(s,a)$ and $\hat{\mu}^{UCB}(s,a)$ in Algorithm \ref{alg:attackRLUninnf} in episode $t$ at step $h$. 

Using Lemma \ref{lemma:ucbLcbRL} and \eqref{eq:con111}, we have that
\begin{equation}\label{eq:eventComplementProb}
    \mathbb{P}\big(\bar{\mathcal{E}}_2\big)\leq 1/(HSAT).
\end{equation}

If event $\mathcal{E}_2$ occurs, then for all $1<t\leq T$ and $h\leq H$, we have that
\begin{equation}\label{eq:uppbound1}
\begin{split}
    &\hat{\mu}_{t,h}^{LCB}(s_{t}(h),\pi_h^+(s_t(h)))+(H-h) \min_{s,a\in \mathcal{S}\times \mathcal{A}}\hat{\mu}_{t,h}^{LCB}(s,a)\\
    &\leq {\mu}(s_{t}(h),\pi_h^+(s_t(h)))+(H-h)\min_{s,a\in S\times A}{\mu}(s,a),\\
    &\leq \tilde{Q}^{\pi^+}_h(s_t(h), \pi_h^+(s_t(h))).
\end{split}
\end{equation}
%\begin{equation}%
%\begin{split}
%     &\min_{a\in \pi^+(s_t(h))} \hat{\mu}^{LCB}(s_{t}(h),a) + (H-h)  \min_{s,a\in S\times A}\hat{\mu}^{LCB}(s,a)\\
%     &\leq \min_{a\in \pi^+(s_t(h))} {\mu}(s_{t}(h),a)+(H-h)\min_{s,a\in S\times A}{\mu}(s,a),\\
%     &\leq \min_{a\in \pi^+(s_t(h))} Q_h(s_t(h),a).
%\end{split}
%\end{equation}
Also, if event $\mathcal{E}_2$ occurs, then for all $1<t\leq T$ and $h\leq H$, we have that
\begin{equation}\label{eq:uppbound2}
\begin{split}
    (H-h) \max_{s,a\in \mathcal{S}\times \mathcal{A}}\hat{\mu}_{t,h}^{UCB}(s,a)&\geq (H-h)\max_{s,a\in S\times A}{\mu}(s,a)\\
    %&\geq \max_{s,a\in S\times A}Q_{h+1}(s,a)\\
    &\geq \mathbb{E}_{s^\prime\sim \mathbb{P}(s^\prime|s_{t}(h),a_t(h))}[\tilde V_{h+1}^{\pi^+}(s^\prime)]. 
\end{split}
\end{equation}
Now, combining \eqref{eq:uppbound1} and \eqref{eq:uppbound2}, under event $\mathcal{E}_2$,  for all $1<t\leq T$ and $h\leq H$, we have that
\begin{equation}\label{eq:keyResult1}
\begin{split}
    &r^o_t(s_t(h),a_t(h))\\
    &= \hat{\mu}^{LCB}_{t,h}(s_{t}(h),\pi_h^+(s_t(h))) + (H-h)  \min_{s,a\in S\times A}\hat{\mu}^{LCB}_{t,h}(s,a)-(H-h)\max_{s,a\in S\times A}\hat{\mu}^{UCB}_{t,h}(s,a)-\epsilon\\
    &\leq  \tilde{Q}^{\pi^+}_h(s_t(h), \pi_h^+(s_t(h)))-\mathbb{E}_{s^\prime\sim \mathbb{P}(s^\prime|s_{t}(h),a_t(h))}[\tilde V_{h+1}^{\pi^+}(s^\prime)]-\epsilon.
    %&\leq \min_{a\in \pi^+(s_t(h))} Q_h(s_t(h),a)-\mathbb{E}_{s^\prime\sim P(s^\prime|s,a_t(h))}[\max_{a}Q^\prime_{h+1}(s^\prime,a)]-\epsilon.
\end{split}
\end{equation}
Now, similar to \eqref{eq:valueFunct}, using \eqref{eq:keyResult1}, we can show that under event $\mathcal{E}_2$, we have
\begin{equation}\label{eq:valueFunct_1}
    V_h^{\pi^+}(s)=\sup_{\pi} V_h^\pi (s).
\end{equation}
Additionally, similar to \eqref{eq:deltaLB}, we also have that under event $\mathcal{E}_2$, we have
\begin{equation}\label{eq:deltaLB_2}
    \Delta(a)=\min_{s,h,\pi} \tilde{Q}_h^{\pi^+}(s,\pi^+_h(s))-\tilde{Q}_h^{\pi}(s,a)\geq \epsilon. 
\end{equation}
Now, using \eqref{eq:deltaLB_2}, under event $\mathcal{E}_2$, we have that with probability $1-\delta$, 
\begin{equation}\label{eq:subOpttimal_1}
\begin{split}
        \sum_{t=1}^{T}\sum_{h=1}^H\epsilon\mathbf{1}(a_t(h)\neq \pi_h^+(s_t(h)))&\leq R^{\mathcal{A}}(T,H) =\tilde O(\sqrt{T}H^\alpha).
\end{split}
\end{equation}
%This implies that 
%\begin{equation}
%\begin{split}
%    R^{\mathcal{A}}(T,H)&= \sum_{t=T^*}^{T}V_1(s_1(t)-V_1^{\pi_t}(s_1(t)),\\
%    &\geq \epsilon \sum_{t=T^*}^T\sum_{h=1}^H\mathbf{1}(a_t(h)\notin \pi^+(s)),
%\end{split}
%\end{equation}
%which implies that with probability $1-\delta$
Thus, combining \eqref{eq:eventComplementProb} and \eqref{eq:subOpttimal_1}, we have that with probability $1-\delta-1/(THSA)$, 
\begin{equation}\label{eq:NumAttUB_2}
      \sum_{t=1}^{T}\sum_{h=1}^H\mathbf{1}(\epsilon_{t,h}(s_t(h),a_t(h)\neq 0)\leq \tilde O(\sqrt{TH}/\epsilon).
\end{equation}
The amount of contamination is 
\begin{equation}
\begin{split}
&|r^o_t(s_t(h),a_t(h))-r_t(s_t(h),a_t(h))|\\
     &\stackrel{(a)}{\leq} |r^o_t(s_t(h),a_t(h))-\mu(s_t(h),a_t(h))|+\sigma \sqrt{4\log(2HSAT)}\\
     &\leq 
     |\mu(s_t(h),a_t(h))|+|r^o_t(s_t(h),a_t(h))|+\sigma \sqrt{4\log(2HSAT)},\\
     &\stackrel{(b)}{\leq}  |\mu(s_t(h),a_t(h))|+ (2H+1)\max_{s,a}|\mu(s,a)|+\epsilon+(4H+3)\sigma \sqrt{4\log(2HSAT)}+2(H+1)M,\\
     &\leq (2H+2)\max_{s,a}|\mu(s,a)|+\epsilon+(4H+3)\sigma \sqrt{4\log(2HSAT)}+2(H+1)M,
\end{split}
\end{equation}
where $(a)$ follows from \eqref{eq:con111-2}, and $(b)$ follows under event $\mathcal{E}_2$. This implies that the total amount of contamination  is
\begin{equation}\label{eq:bounpart2}
     \sum_{t=1}^{T}\sum_{h=1}^H|\epsilon_{t,h}(s_t(h),a_t(h)|= \tilde O\bigg(\sqrt{T}H^{\alpha}(HM+\epsilon+H\sigma\sqrt{\log(HTSA)})/\epsilon\bigg),
\end{equation}
%\hf{should the first term be $\frac{\sqrt{T}H^{\alpha}}{\epsilon}$ instead?}\ar{Thats right. I have written the numerator first and then the denominator. } 
with probability $1-\delta-2/(HSAT)$. 

Using \eqref{eq:NumAttUB_2}, we have that part $2.$ follows.
Also, using \eqref{eq:bounpart2}, we have that part $3.$ follows. Finally, we have that with probability $1-\delta$
\begin{equation}
\begin{split}
       \sum_{t=1}^T\sum_{h=1}^H\mathbf{1}(a_t(h)= \pi_h^+(s_{t}(h)))&=TH-\sum_{t=1}^T\sum_{h=1}^H\mathbf{1}(a_t(h)\neq \pi_h^+(s_{t}(h)))\\
       &=\Omega(T),
\end{split}
\end{equation}
where the last equality follows from \eqref{eq:subOpttimal_1}. The statement of the theorem follows. 
\end{proof}

\section{Potential Societal Impact}
 In this paper, we evaluate the security threats to reinforcement learning (RL) algorithms. We aim
 to understand the regime where and how these algorithms can be attacked, thus revealing their vulnerability. This
 might limit the usage of these RL algorithms in sensitive application domains like cyber-physical systems. However,
 understanding the vulnerability can also help in designing better defense techniques in the future.

 %\cite{jin2018q} show that the lower bound on the regret of any learning algorithm is $\tilde{\Omega}(\sqrt{T})$. This along with \cite[Theorem 1]{rakhsha2020policy} implies that the amount of contamination and number of attacks are lower bounded by $\tilde{\Omega}(\sqrt{T})$. \hf{\cite{rakhsha2020policy} studied a special type of attack and proved a lower bound cost for their attack in thm 1 -- why the lower bound cost for that special attack implies a lower bound for any possible attack? } Thus, according to Theorem \ref{thm:unBoundedWhiteBox}, we have that the attack proposed in \eqref{eq:attackStrategy_RL_fullInfo2} is order-optimal, upto logarithmic factors, in $T$ with high probability. 

\end{document}